\documentclass{article} 
\usepackage{iclr2026_conference,times}
\iclrfinalcopy

\usepackage{hyperref}
\usepackage{url}

\usepackage{hyperref}       
\usepackage{url}            
\usepackage{booktabs}       
\usepackage{amsfonts}       
\usepackage{nicefrac}       
\usepackage{microtype}      
\usepackage{xcolor}         

\usepackage{natbib}  

\usepackage{amsmath}
\usepackage{amssymb}
\usepackage{mathtools}
\usepackage{amsthm}
\usepackage{verbatim} 
\usepackage{subcaption}
\usepackage{algorithm}
\usepackage{algorithmic}
\usepackage{multirow} 
\usepackage{wrapfig}

\theoremstyle{plain}
\newtheorem{theorem}{Theorem}[section]
\newtheorem{proposition}[theorem]{Proposition}

\newtheorem{corollary}[theorem]{Corollary}
\theoremstyle{definition}
\newtheorem{definition}[theorem]{Definition}
\newtheorem{assumption}[theorem]{Assumption}
\theoremstyle{remark}

\title{Dynamic Correction of Erroneous State Estimates via Diffusion Bayesian Exploration}

\author{Yiwei Shi $^1$, Hongnan Ma $^1$, Mengyue Yang $^1$, Cunjia Liu $^2$, Weiru Liu $^1$\\
$^1$ University of Bristol, $^2$ Loughborough University
}

\begin{document}

\maketitle

\begin{abstract}
In emergency response and other high-stakes societal applications, early-stage state estimates critically shape downstream outcomes. Yet, these initial state estimates—often based on limited or biased information—can be severely misaligned with reality, constraining subsequent actions and potentially causing catastrophic delays, resource misallocation, and human harm. Under the stationary bootstrap baseline (zero transition and no rejuvenation), bootstrap particle filters exhibit Stationarity-Induced Posterior Support Invariance (S-PSI), wherein regions excluded by the initial prior remain permanently unexplorable, making corrections impossible even when new evidence contradicts current beliefs. While classical perturbations can in principle break this lock-in, they operate in an always-on fashion and may be inefficient. To overcome this, we propose a diffusion-driven Bayesian exploration framework that enables principled, real-time correction of early state estimation errors. Our method expands posterior support via entropy-regularized sampling and covariance-scaled diffusion. A Metropolis–Hastings check validates proposals and keeps inference adaptive to unexpected evidence. Empirical evaluations on realistic hazardous-gas localization tasks show that our approach matches reinforcement learning and planning baselines when priors are correct. It substantially outperforms classical SMC perturbations and RL-based methods under misalignment, and we provide theoretical guarantees that DEPF resolves S-PSI while maintaining statistical rigor.

\end{abstract}

\addtocontents{toc}{\protect\setcounter{tocdepth}{-1}}

\section{Introduction}

In disaster emergency management, early decisions play a pivotal role in shaping the outcome of a crisis \cite{steel2010bayesian,kruschke2010bayesian}. An initial state estimate—for example, an early guess of a hazardous leak’s location—is often made with scant data under intense time pressure. Such early estimates carry great weight in guiding subsequent actions, yet they are highly vulnerable to error due to uncertainty and human subjective judgment. If the initial assumption is mistaken, the entire response effort can become locked into a false premise, failing to adjust even as new data arrives. This lock-in effect in high-stakes scenarios leads to catastrophic delays, misallocation of resources, and heightened risks to society. The fragility of such initial state assumptions thus poses a serious challenge for social decision-making processes, revealing the need for methods that can rapidly correct initial mistakes in real time. 

One major source of initial state estimation error is the inherent uncertainty and bias in early-stage human assessments. When decision-makers narrow their focus based on incomplete or misleading information, they risk excluding the true state of the world from consideration. Once an erroneous initial state estimate anchors the process, it creates a path-dependent trap: monitoring systems and responders may remain fixated on the wrong location or strategy, even as contradictory evidence mounts. In practice, this means a search algorithm could ignore critical regions outside the presumed hazard zone, or resources might continue to be deployed ineffectively, compounding the crisis’s social and economic impacts. To avoid such outcomes, it is crucial to develop adaptive algorithms that revisit and revise early assumptions as new observations arrive. 

From a robotics and Bayesian inference perspective, the above scenario is essentially a state estimation problem: an agent must form an initial belief (prior) about the latent state (e.g., the hazard’s location) and update this belief as new sensor data arrives. If the initial belief is mis-specified, the goal is to correct this incorrect state estimate in real time by incorporating incoming evidence. For clarity, we use the term initial policy error to denote such a mistaken initial state estimate under prior uncertainty, which our approach aims to correct. In recent years, particle filtering (PF) has become a go-to approach for sequential Bayesian state estimation in these settings \cite{ristic2013particle,gordon1993novel,doucet2001introduction}. PF offers a principled way to update beliefs by fusing sensor data with prior knowledge, making it a representative framework to tackle the problem of erroneous initial state estimates in emergency response and other high-stakes decision domains. However, existing Bayesian filtering methods \cite{fox2003bayesian,smidl2008variational} struggle with erroneous initial state estimates, often failing to recover when the true state lies outside the initial belief. 1) Bootstrap particle filters \cite{candy2007bootstrap,gordon1995bayesian} tend to remain confined to the support of the initial prior, hindering exploration beyond the originally assumed region when operated under a stationary bootstrap baseline with zero transition and no rejuvenation. We formalize this baseline-specific lock-in as Stationarity-Induced Posterior Support Invariance (S-PSI): under the zero-transition, no-rejuvenation baseline, the posterior cannot escape the initial support. Importantly, S-PSI is not an inherent limitation of PF; classical countermeasures (e.g., jittering, roughening) can in principle expand support (we include these as baselines in our experiments). 2) More advanced PF variants,such as auxiliary particle filters \cite{mountney2009bayesian,branchini2021optimized} or filters with optimal proposal distributions, can partially mitigate bias when the true state has a small non-zero prior probability \cite{fox2001kld,douc2005comparison,liu1998sequential,doucet2000sequential,arulampalam2002tutorial}, but they fundamentally cannot handle cases where the true state was assigned zero initial probability. Under the standard Bayesian update, any region with zero prior mass will remain at zero posterior mass indefinitely, meaning the algorithm can never discover a completely excluded possibility under this baseline. 3) Attempts to address these issues by augmenting particle filters have had limited success. Some works inject noise or broaden the prior artificially, and others integrate reinforcement learning (RL) with PF \cite{shi2024autonomous,zhao2022deep,park2022source} to actively guide sensor exploration. While such approaches can improve data collection, they may inherit the same blind spots from a mis-specified prior and often introduce significant complexity and resource demands. Without a new perspective, Bayesian trackers and decision methods remain at risk of locking onto an incorrect initial belief, especially under the S-PSI. 

To overcome these challenges, we propose a novel approach called Diffusion-Enhanced Particle Filtering (DEPF) that dynamically corrects erroneous initial state estimates via a diffusion-driven Bayesian exploration mechanism. The key insight of DEPF is to expand the particle filter’s support in response to observation feedback, allowing the algorithm to break out of the constraints imposed by a flawed initial prior. Instead of passively accepting the prior’s limits, our method systematically injects a small number of exploratory particles into regions outside the currently believed range. This injection is guided by indicators of model inconsistency when incoming sensor data strongly contradicts the filter’s predictions (e.g., high error or entropy). A controlled stochastic diffusion process then spreads these exploratory particles into previously neglected areas, effectively probing the hypothesis that the true state might lie beyond the old bounds. We incorporate a Bayesian validation step to ensure that the expanded support remains statistically coherent. Through this belief-triggered diffusion-and-validation cycle, DEPF augments the PF inference layer and mitigates S-PSI when it arises under the stationary bootstrap baseline. 

The main contributions of this work are as follows:
(1) We identify and formally define the Stationarity-Induced Posterior Support Invariance (S-PSI) under the zero-transition, no-rejuvenation bootstrap baseline, characterizing it as a diagnostic condition rather than a universal PF limitation (\S \ref{subsec:PSI}).
(2) We propose the DEPF framework, a particle filtering method that introduces a principled, belief-triggered technique to dynamically expand inference support beyond initial belief constraints (\S \ref{sec:depf}).
(3) We demonstrate via theory and experiments (hazardous gas leak scenarios) that DEPF can effectively correct initial state estimation errors across different scales and error severities, substantially improving localization and response efficiency over existing methods, including RL/planning baselines and classical perturbations (\S \ref{sec:Experiment}).

\section{Related Work}

In emergency localization scenarios \cite{wu2021research,hite2019bayesian}, \textit{Bayesian filtering} \cite{fox2003bayesian,smidl2008variational,quinlan2009multiple} methods such as the bootstrap particle filter leverage Bayesian inference to iteratively update state estimates but typically assume correctly-specified initial priors, thus \emph{may become ineffective under severely misaligned early assumptions when operated under a stationary bootstrap baseline (zero transition, no rejuvenation), due to} \textit{Stationarity-Induced Posterior Support Invariance (S-PSI)} \cite{gordon1995bayesian,candy2007bootstrap}. \emph{S-PSI is a baseline-specific diagnostic rather than an inherent limitation of PF; classical countermeasures such as jittering/roughening or resample--move can, in principle, expand support, and we include them as baselines in \S \ref{sec:Experiment}.} Advanced particle filter variants, including \textit{auxiliary particle filters} \cite{mountney2009bayesian,branchini2021optimized} and filters using \textit{optimal proposal distributions}, improve proposal quality and sample efficiency but still fail when the initial belief assigns \emph{zero} prior probability to the true state \cite{fox2001kld,arulampalam2002tutorial}, i.e., in the presence of a zero-prior barrier without explicit support expansion. Meanwhile, \textit{information-theoretic methods} like \textit{Infotaxis} \cite{vergassola2007infotaxis}, \textit{Entrotaxis} \cite{hutchinson2018entrotaxis} and \textit{DCEE} \cite{chen2021dual} focus sensor motions on maximizing expected information gain or reducing entropy, yet they typically operate within the belief support induced by the initial prior and thus cannot systematically correct severe prior misalignment in real time without a support-expanding inference layer. More recently, integrated \textit{reinforcement learning and particle filtering (RL-PF)} methods have emerged—e.g., \textit{AGDC} and its variants using KL-divergence or entropy-based intrinsic rewards \cite{shi2024autonomous}, \textit{PC-DQN} \cite{zhao2022deep}, and \textit{GMM-PFRL} \cite{park2022source}. While these RL-driven methods exhibit stronger adaptive exploration, they can still inherit the zero-prior barrier when the underlying filtering layer does not expand support (cf. the S-PSI baseline). In contrast, our proposed \textit{DEPF} explicitly addresses this gap by introducing a belief-triggered, validated support-expansion mechanism that operates at the inference layer, making it complementary to proposal-improvement filters, classical perturbations (jittering/roughening/rejuvenation), and planning/RL controllers, and yielding superior robustness under severely misaligned early assumptions.

\section{Problem Formulation and Preliminaries}
\label{sec:Preliminaries}
\subsection{Problem Setup}
\label{subsec:Problem_Setup}
Consider a two-dimensional spatial domain \(\Omega \subset \mathbb{R}^2\) with a stationary hazardous gas source. We describe the unknown source term by the parameter vector at time step \(k\):
$\Theta_k = [x_s, y_s, q_s, u_s, \phi_s, d_s, \tau_s]^\top \in \mathbb{R}^7$
where \( \boldsymbol{p}_s = (x_s, y_s) \in \Omega \subset \mathbb{R}^2 \) represent the Cartesian coordinates of the source, \( q_s \in \mathbb{R}^{+} \) denotes the scaled release strength, representing the true emission rate adjusted by an unknown sensor calibration factor, \( u_s \in \mathbb{R}^{+} \) and \( \phi_s \in [0, 2\pi) \) represent the wind speed and wind direction respectively, \( d_s \in \mathbb{R}^{+} \) describes the diffusivity of the gas in air, \( \tau_s \in \mathbb{R}^{+} \) indicates the effective lifetime of the gas. At each discrete time step \(k\), a mobile robot equipped with a gas sensor occupies position \(\boldsymbol{p}_k =(x_k,y_k) \in \Omega\) and records a scalar sensor output \( z_k \in \mathbb{R}^+ \), which represents the raw \textit{voltage signal} from the gas sensor. This signal serves as the \textit{observation} in the Bayesian filtering model, linking sensor data to the hidden source state \( \Theta \). The cumulative sensor readings up to step \(k\) are denoted as \( z_{1:k} = \{ z_1, z_2, \dots, z_k \} \). \textcolor{cyan}{\textbf{Our objective}} is to estimate the posterior distribution \( p(\Theta_k \mid z_{1:k}) \) given the observed sensor signals \( z_{1:k} \) (i.e., raw voltage measurements), and robot locations, under the assumption of source stationarity, i.e., \(\Theta_{k+1} = \Theta_k\). To handle nonlinearities and intermittency in sensor readings, a particle filter is adopted to iteratively approximate this posterior.

\subsection{Observation Model}
\label{sec:Ob-model}
We adopt a simplified analytical plume model derived from the advection–diffusion equation to represent gas transport from the source to the sensor location. The expected sensor output at location \(p_k\) under source parameters \(\Theta\) is given by:
$h(p_k; \Theta) = \frac{q_s}{4\pi d_s \|p_k - p_s\|} 
\cdot \exp\left(-\frac{\|p_k - p_s\|}{\lambda}
 -\frac{\psi}{2 d_s} \right)$ with $\psi = (x_k - x_s) u_s \cos\phi_s + (y_k - y_s) u_s \sin\phi_s$ and $\lambda = \sqrt{{d_s \tau_s}/[{1 + ({u_s^2 \tau_s}/{4 d_s})}]}$. 

Since we use a low-cost metal oxide (MOX) sensor, the sensor output is not a calibrated concentration but a voltage value subject to significant uncertainty and miss-detection. Thus, the final measurement model is:
$ z_k = D_k \cdot \big(h(p_k;\Theta) + \bar{v}_k\big) + (1 - D_k)\cdot v_k. $ where \( \bar{v}_k \sim \mathcal{N}(0, \bar{\sigma}_k^2) \) is additive sensor noise, variance $\bar\sigma_k^{2}$, \( v_k \sim \mathcal{N}(0, \sigma_k^2)\) is background noise in clean air, variance $\sigma_k^{2}$, \( D_k \in \{0,1\} \sim \text{Bernoulli}(P_d) \) encodes whether the sensor successfully detects the gas, \(P_d\) is the probability of detection, reflecting turbulence, dilution, or sensor failure. The resulting \textit{Gaussian mixture likelihood function} becomes:
$ p(z_k \mid \Theta) =
(1 - P_d) \cdot \mathcal{N}(z_k;\, 0,\, \sigma_k^2)
+ P_d \cdot \mathcal{N}(z_k;\, h(p_k; \Theta),\, \bar{\sigma}_k^2) $


\subsection{Sequential Particle Filtering}
\label{sec:pf-brief}

Particle filtering offers a non-parametric, sequential Bayes estimator of the posterior
$p(\Theta_k\mid z_{1:k})$.
We represent that posterior by a set of \(N\) weighted particles
\(\{\,\Theta_k^{(i)}, w_{k}^{(i)}\}_{i=1}^{N}\),
$p(\Theta_k\mid z_{1:k})
\;\approx\;
\sum_{i=1}^{N}
w_{k}^{(i)}\,\delta~\!\bigl(\Theta_k-\Theta_k^{(i)}\bigr),
\sum_{i=1}^{N}w_{k}^{(i)}=1$,
with \(\delta(\cdot)\) the Dirac delta. 
In general, particles propagate via a transition kernel
$p(\Theta_k\mid\Theta_{k-1})$ and are sampled from a proposal
$q(\Theta_k\mid \Theta_{k-1}, z_k)$.
The importance weights then update as
$
\tilde w_k^{(i)}
= w_{k-1}^{(i)}\,
\frac{p\!\bigl(z_k \mid \Theta_k^{(i)}\bigr)\,
      p\!\bigl(\Theta_k^{(i)}\mid \Theta_{k-1}^{(i)}\bigr)}
     {q\!\bigl(\Theta_k^{(i)}\mid \Theta_{k-1}^{(i)}, z_k\bigr)},
w_k^{(i)} = \tilde w_k^{(i)} \Big/ \sum_{j=1}^N \tilde w_k^{(j)},
$
with the likelihood $p(z_k\mid\Theta)$ given in \S\ref{sec:Ob-model}.
We trigger resampling when the effective sample size
$\mathrm{ESS}_k=1/\sum_{i=1}^N (w_k^{(i)})^2$ falls below a threshold $\eta$. In the widely used \emph{bootstrap filter}, the proposal equals the transition, $q(\Theta_k\mid \Theta_{k-1}, z_k) = p(\Theta_k\mid \Theta_{k-1})$, so the weight update simplifies to $\tilde w_k^{(i)} = w_{k-1}^{(i)}\,p\!\bigl(z_k\mid \Theta_k^{(i)}\bigr)$. In our setting the source term $\Theta$ is assumed \emph{static during the response horizon}. Without a natural dynamical law, a common and widely used reference method is the \emph{bootstrap filter}, where particles are simply carried forward and only their weights are updated by the likelihood of new sensor data.

\subsection{Stationarity-Induced Posterior Support Invariance (S-PSI)}
\label{subsec:PSI}

The simplicity of the bootstrap filter, while natural under static parameters, reveals a structural vulnerability: particles remain fixed in parameter space and can never leave the initial prior region. Even when new observations strongly contradict the prior, the filter cannot escape this confinement. We formalize this lock-in effect as 
\emph{Stationarity-Induced Posterior Support Invariance (S-PSI)}.

Let the prior be $p_0(\Theta)$ with support
$\mathcal{S}_{\text{prior}} := \operatorname{supp} p_{0}(\Theta) = \{\,\Theta: p_0(\Theta)>0\,\} \subset \mathbb{R}^{7}$.

\begin{assumption}[S0: zero transition, no rejuvenation] \label{assump:S0}
$ p(\Theta_k\mid \Theta_{k-1})
= \delta(\Theta_k-\Theta_{k-1})$
and no rejuvenation step (e.g., jittering, roughening, resample–move) is applied.
\end{assumption}

\begin{proposition}[S-PSI under \ref{assump:S0}]
If particles are initialized within $\mathcal{S}_{\text{prior}}$, then for all $k$,
$\operatorname{supp}\!\bigl(p(\Theta\mid z_{1:k})\bigr)
\subseteq \mathcal{S}_{\text{prior}}.$
In words, the posterior support remains permanently trapped inside the initial prior region. As a direct consequence, if the true source $\Theta^\ast$ lies outside the prior, then $\Theta^\ast \notin \mathcal{S}_{\text{prior}}
\;\Rightarrow\; p(\Theta^\ast \mid z_{1:k}) = 0,\ \forall k ,$ i.e., the filter fundamentally cannot discover it, not due to likelihood mismatch, but simply because no particles ever enter that excluded region.
\end{proposition}

\subsection{POMDP formulation for belief‑aware sensor planning}
\begin{wrapfigure}{r}{0.30\linewidth} 
  \centering
  \includegraphics[width=\linewidth]{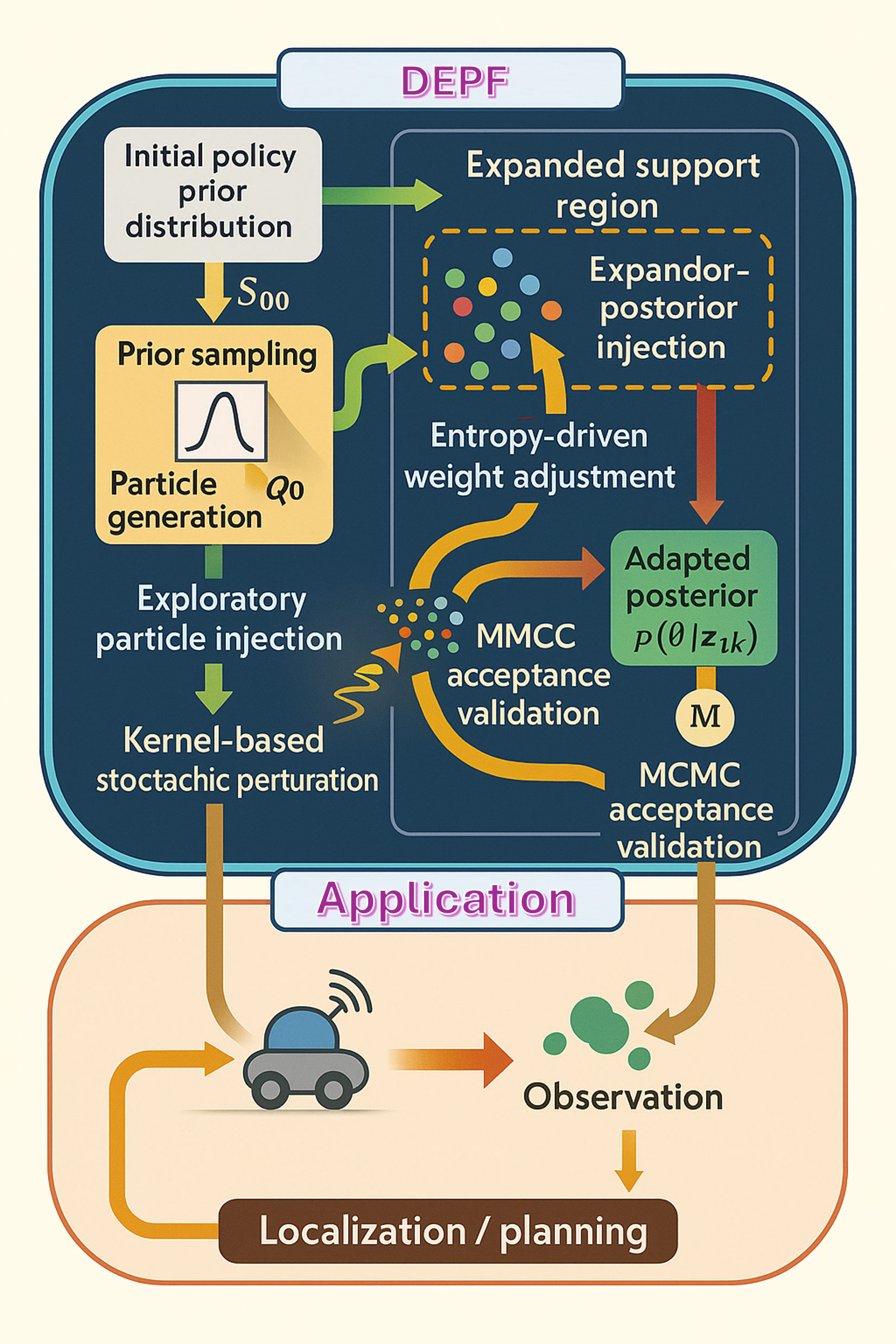}
  \caption{Flowchart of DEPF.}
  \label{fig:Flowchart}
\end{wrapfigure}
The interaction between the mobile robot and the unknown gas plume can be framed as a \emph{partially observableMarkov decision process} (POMDP) \(\mathcal{M}=(\mathcal{S},\mathcal{A},\Omega,T,O,R,\gamma)\), where \(\mathcal{S}\) is the latent state space containing the stationary source vector \(\Theta=(x_s,y_s,q_s,u_s,\phi_s,d_s,\tau_s)^\top\); \(\mathcal{A}\) is the set of motion commands that move the robot in the plane; \(\Omega\) is the observation space, where an observation at time \(k\) is \(o_k=(p_k,z_k)\) with \(p_k\in\mathbb{R}^2\) the robot position and \(z_k\in\mathbb{R}\) the sensor voltage.  The transition kernel factorises as \(T((p',\Theta')\mid(p,\Theta),a)=T_p(p'\mid p,a)\,\delta(\Theta'-\Theta)\): robot kinematics are deterministic while the source parameters remain unchanged.  The observation model is \(O(o_k\mid(p_k,\Theta),a_{k-1})=p(z_k\mid p_k,\Theta)\), where the likelihood \(p(z_k\mid p_k,\Theta)\) is the mixture‑Gaussian plume sensor model of \S \ref{sec:Ob-model}.  At decision time the agent cannot access \(\Theta\) directly, so it reasons with the belief \(b_k=p(\Theta\mid z_{1:k})\) supplied by the particle filter.  We therefore feed the RL policy with the augmented information state \(s^{\text{RL}}_k=(p_k,z_k,b_k)\).  The reward is chosen as the expected one‑step information gain \(R_k=\mathbb{E}_{o_{k+1}}[D_{\mathrm{KL}}(b_{k+1}\,\|\,b_k)]\), encouraging actions that shrink posterior uncertainty, and future rewards are discounted by \(\gamma\in(0,1)\).  The objective is to learn a policy \(\pi^\ast\) that maximises the expected discounted return \(J=\mathbb{E}[\sum_{t=0}^{\infty}\gamma^{t}R_{k+t}]\), thereby steering the robot along paths that are most informative about the hidden source. The details are provided in Appendix \S~\ref{sec:rl-controller}.

\section{Diffusion-Driven Support Range Expansion}
\label{sec:depf}

Particle filtering provides a sequential Bayesian framework with weighted particles approximating the posterior. 
Under the \emph{S-PSI baseline} introduced in \S \ref{sec:PSI}—i.e., a \textbf{zero-transition, no-rejuvenation} bootstrap setting—the posterior support cannot escape the initial prior support $\mathcal{S}_{\text{prior}}$. 
We treat this as a \emph{didactic baseline}, not an inherent limitation of PF. 
To mitigate S-PSI \emph{when it arises}, we propose a diffusion-enhanced correction module (Fig.~\ref{fig:Flowchart}) that \textbf{(i)} injects a small fraction of exploratory particles, \textbf{(ii)} applies covariance-scaled stochastic diffusion, and \textbf{(iii)} validates proposals via Metropolis--Hastings (MH), thereby enabling \emph{minimal-bias}, data-triggered support expansion.

\textbf{Adaptive Diffusion via Exploratory Particles:} At each time step, a subset of particles is designated as \textit{exploratory particles}, which introduce a uniform diffusion process into the framework. These particles are sampled from an adaptively extended bounding region $\mathcal{B}_k$, dynamically adjusted according to the current particle distribution to cover regions beyond the initial prior boundary $\mathcal{S}_{\text{prior}}$: $\mathcal{B}_k = [0, x_{\text{max}} + \delta] \times [0, y_{\text{max}} + \delta]$, where $x_{\text{max}}$ and $y_{\text{max}}$ denote the current maximum particle positions along each spatial dimension and $\delta$ is an adaptively determined margin parameter. Exploratory particles are then uniformly sampled from this expanded bounding region: $\Theta_k^{(j)} \sim \mathcal{U}(\mathcal{B}_k), j \in \mathcal{E}$, where $\mathcal{E}$ represents the indices of exploratory particles. The exploratory particles are initialised with small weights:
$w_k^{(j)} = \frac{\epsilon}{|\mathcal{E}|}, \epsilon \ll 1.$
This mechanism enables the bootstrap filter to sample states outside the original support range $\mathcal{S}_{\text{prior}}$, thereby increasing the likelihood of reaching states $\Theta^* \notin \mathcal{S}_{\text{prior}}$.

\textbf{Entropy-Driven Diffusion Regularisation:} To ensure that the exploratory diffusion does not collapse prematurely, an entropy regularisation term is added during the weight update step. This regularisation diffuses the weights across all particles, encouraging exploration of low-probability regions: $ w_k^{(i)} \leftarrow w_k^{(i)} + \beta \cdot H(w_k)$, where $H(w_k)$ is the entropy of the weight distribution, defined as: $ H(w_k) = -\sum_{i=1}^N w_k^{(i)} \log(w_k^{(i)} + \epsilon)$. The regularisation parameter $\beta$ is adaptively chosen based on the discrepancy between the current entropy and a predefined target entropy $H_{\text{target}}$:
$\beta = \max\left(\beta_{\min}, \min\left(\beta_{\max}, \frac{H_{\text{target}} - H(w_k)}{H_{\text{target}}}\right)\right)$, where $\beta_{\min}$ and $\beta_{\max}$ represent the predefined minimum and maximum regularisation strengths, respectively. By penalising weight distributions that become overly concentrated, this adaptive entropy-based mechanism promotes balanced diffusion across the state space. The diffusion of weights helps exploratory particles retain influence, thus effectively encouraging the discovery and sustained exploration of regions beyond $\mathcal{S}_{\text{prior}}$.

\textbf{Kernel-Induced Stochastic Diffusion:} To further expand the particle support range dynamically, we introduce a stochastic diffusion mechanism based on kernel perturbations. Each particle \( \Theta_k^{(i)} \) is perturbed by a Gaussian kernel that models diffusion within the local neighbourhood:
$\Delta \Theta_k^{(i)} \sim h_{\text{opt}} \cdot \mathcal{L} \cdot \mathcal{N}(0, I)$, where \( h_{\text{opt}} = A \cdot N^{-\frac{1}{n+4}} \) is the optimal kernel bandwidth dynamically adjusted to balance exploration and precision \( \mathcal{L} \) is the lower triangular matrix obtained from the Cholesky decomposition of the covariance matrix \( \Sigma \), ensuring diffusion adapts to the local particle distribution. The covariance matrix \( \Sigma \) is computed dynamically: $\Sigma = \sum_{i=1}^N w_k^{(i)} (\Theta_k^{(i)} - \mu)(\Theta_k^{(i)} - \mu)^T + \lambda I$, where \( \mu = \sum_{i=1}^N w_k^{(i)} \Theta_k^{(i)} \) is the weighted mean, and \( \lambda > 0 \) ensures positive definiteness of \( \Sigma \).

This perturbation mechanism expands the effective support range by introducing stochastic diffusion, allowing particles to explore new regions iteratively:
$\Theta_k^{(i)} \leftarrow \Theta_k^{(i)} + \Delta \Theta_k^{(i)}.$

\textbf{Diffusion-Driven Validation via MCMC:} To ensure consistency with the target posterior distribution, a Metropolis-Hastings acceptance criterion \cite{hastings1970monte} validates the diffused particles. For each perturbed particle \( \Theta_k^{(i)} \), the acceptance probability is: $\alpha_i = \frac{w_{\text{new}}^{(i)}}{w_{\text{old}}^{(i)}} \cdot \exp\left(-\frac{1}{2} \Delta \Theta_k^{(i)^T} \Sigma^{-1} \Delta \Theta_k^{(i)}\right).$
A uniformly sampled random variable \( u_i \sim \mathcal{U}(0, 1) \) determines whether the particle is accepted: $ \Theta_k^{(i)} = \Theta_k^{(i)} - \Delta \Theta_k^{(i)} \cdot \mathbb{I}(\alpha_i < u_i)$, where $\mathbb{I}(\alpha_i < u_i)$ is the indicator function, which equals 1 when $\alpha_i < u_i$ is true, and 0 otherwise. This step ensures that the diffusion-driven expansion aligns with the posterior distribution, preserving the accuracy of the particle filter.

\textbf{Diffusion-Enhanced Particle Filtering:} By integrating exploratory particles, entropy-driven diffusion regularisation, and kernel-induced stochastic perturbations, the proposed framework creates a dynamic diffusion process that iteratively expands the effective support range. The recursive relationship for the support range becomes:
$\mathcal{S}_{k+1} = (\mathcal{S}_k \cup \mathcal{B}) \oplus h_{\text{opt}}$,
where \( \oplus h_{\text{opt}} \) represents kernel-induced stochastic diffusion. This diffusion framework effectively mitigates the posterior support invariance by continuously extending the particle filter's exploration capability, enabling robust state estimation for target states \( \Theta^* \notin \mathcal{S}_{\text{prior}} \). The detailed theoretical analysis and justification of the effectiveness of our proposed method are provided in Appendix \S \ref{appendix:theory_analysis}.

\section{Experiment}
\label{sec:Experiment}
To evaluate the ability of our method to dynamically recover from severe prior misalignment, we conduct experiments using the ISLC environments (ISLCenv), a simulation suite designed for emergency gas leak localization under varying levels of initial policy error. As detailed in ~\S\ref{sec:ISLC_env}, ISLCenv models a multi-source Gaussian plume and simulates noisy sensor observations without explicit reward signals. This setup allows us to rigorously assess the capacity of DEPF and competing baselines to overcome posterior support limitations and adaptively \textcolor{cyan}{\textbf{infer the full 7-D parameter vector $\Theta$}} in real time under realistic operational constraints, \textcolor{cyan}{\textbf{rather than only the source coordinates}}.

\subsection{Evaluation Metrics and  Baseline Algorithms}

To evaluate our proposed approach and compare it against baseline algorithms, we use four distinct metrics: \textit{Operational Completion Efficacy (OCE)}, which measures how frequently emergency response missions meet their goals, with higher scores indicating better deployment effectiveness; \textit{Average Deployment Efficiency (ADE)}, representing the average distance traveled by response units, where shorter distances imply more efficient routing; \textit{Response Execution Velocity (REV)}, quantifying the time duration from deployment to task completion, with faster times signifying more efficient operations; and \textit{Localization Precision Score (LPS)}, assessing the accuracy of source localization by computing the average discrepancy between estimated and actual source locations, with smaller values denoting higher accuracy. Our proposed method is evaluated alongside various baseline algorithms, grouped according to their methodological foundations. The \textcolor{cyan}{\textbf{first group}} merges \textcolor{cyan}{reinforcement learning with Bayesian inference} and includes a single representative, \textit{AGDC} \cite{shi2024autonomous}, which leverages the particle-filter RL posterior and uses intrinsic rewards derived from belief updates to guide exploration. Additionally, we include two other reinforcement learning approaches, \textit{PC-DQN} \cite{zhao2022deep} and \textit{GMM-PFRL} \cite{park2022source}, which independently leverage particle filtering parameters as states for RL training. The \textcolor{cyan}{\textbf{second group}} integrates \textcolor{cyan}{planning and Bayesian inference} approaches, represented by \textit{Infotaxis} \cite{vergassola2007infotaxis}, \textit{Entrotaxis} \cite{hutchinson2018entrotaxis}, and \textit{DCEE} \cite{chen2021dual}. Finally, we include a \textcolor{cyan}{\textbf{third group} of classical SMC perturbation baselines that are not subject to S-PSI constraints}: \textit{PF+Jittering} \cite{liu1998sequential,doucet2000sequential}, \textit{PF+Roughening} \cite{gordon1993novel,gordon1995bayesian}, and \textit{PF+Rejuvenation} \cite{hastings1970monte,doucet2000sequential}. Implementation details and hyperparameter grids are given in \S \ref{sec:baselines} and \ref{app:diffusion-pf}.

\subsection{Scenario Parameterization and Evaluation}

\begin{figure}[htbp]
  \centering
  \includegraphics[width=\linewidth]{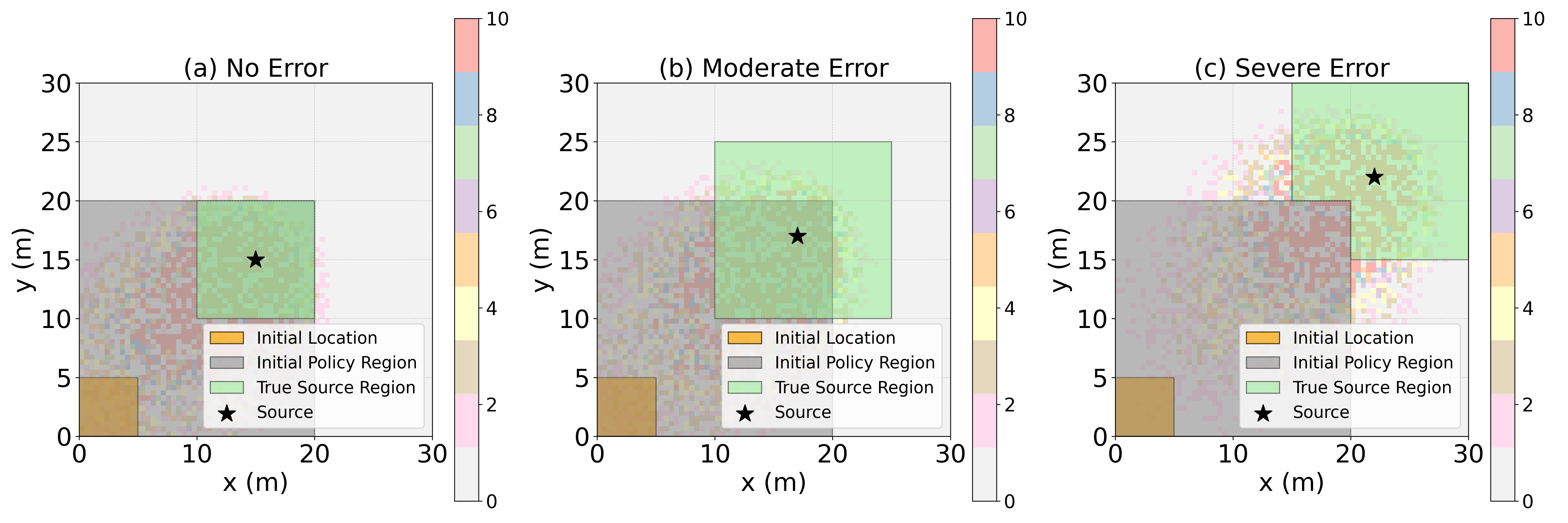}
  \caption{Experimental Scenarios for Policy Errors in Emergency Response.}
  \label{fig:policy_error}
\end{figure}
We evaluate our proposed approach under three distinct scenarios representing different levels of initial policy error made by emergency response planners. In all scenarios, the agents operate in a $30\times30$ spatial domain, and the gas source location, wind speed, and wind direction are sampled according to the distributions specified in Table \ref{tab:training_parameters} of Appendix \S \ref{sec:ISLC_env}. Each training and testing instance uses entirely different plume parameters, with 1000 training and 500 testing instances. Agents start uniformly distributed within the initial subregion $(0,5)\times(0,5)$ and move with unit-length steps.

The three initial estimation error scenarios in Figure~\ref{fig:policy_error} represent increasing levels of prior misalignment between the initial belief and the true disaster source. \textbf{(1) No Error (Ideal)} corresponds to an ideal case where the initial particle distribution fully covers the true source region, providing a best-case baseline without decision uncertainty. \textbf{(2) Moderate Error (Partial Misalignment)} models a realistic situation where the initial assumed region partially overlaps with the true source area, testing whether each method can adapt to moderate initial mistakes. \textbf{(3) Severe Error (Complete Misalignment, PSI)} is the most challenging case, where the initial prior support is entirely disjoint from the true source location, creating a strict zero-prior barrier. This PSI scenario explicitly probes whether an algorithm can expand its posterior support and recover from severely misaligned initial state estimates, a setting in which standard bootstrap particle filters typically fail. The exact spatial layouts and parameter ranges for these three scenarios are provided in the appendix \ref{app:scenario-spec}.

\begin{table}[!t]
\centering
\caption{Performance Comparison under Different Initial Policy Errors and Environment Scales}
\label{tab:policy_error_experiments}
\resizebox{\textwidth}{!}{%
\begin{tabular}{clcccccc}
\toprule
\textbf{Metric} & \textbf{Method} & \multicolumn{3}{c}{\textbf{Small-scale Environment (1:30)}} & \multicolumn{3}{c}{\textbf{Large-scale Environment (1:300)}} \\
\cmidrule(lr){3-5} \cmidrule(lr){6-8}
& & No Error & Moderate Error & Severe Error & No Error & Moderate Error & Severe Error \\
\midrule
\multirow{10}{*}{OCE$\uparrow$} 
& DEPF (ours) & \textbf{0.90±0.03} & \textbf{0.90±0.03} & \textbf{0.89±0.03} & \textbf{0.90±0.03} & \textbf{0.90±0.03} & \textbf{0.88±0.03} \\
\cmidrule{2-8}
& AGDC & 0.90±0.03 & 0.45±0.04 & $<$ 0.05 & 0.87±0.04 & 0.42±0.05 & $<$ 0.05 \\
& PC-DQN & 0.80±0.04 & 0.40±0.04 & $<$ 0.05 & 0.79±0.04 & 0.38±0.04 & $<$ 0.05 \\
& GMM-PFRL & 0.80±0.04 & 0.40±0.04 & $<$ 0.05 & 0.77±0.04 & 0.35±0.05 & $<$ 0.05 \\
\cmidrule{2-8}
& Infotaxis & 0.85±0.04 & 0.25±0.02 & $<$ 0.05 & $<$ 0.05 & $<$ 0.05 & $<$ 0.05 \\
& Entrotaxis & 0.26±0.02 & 0.13±0.02 & $<$ 0.05 & $<$ 0.05 & $<$ 0.05 & $<$ 0.05 \\
& DCEE & 0.62±0.03 & 0.21±0.03 & $<$ 0.05 & $<$ 0.05 & $<$ 0.05 & $<$ 0.05 \\
\cmidrule{2-8}
& PF+Jittering    & 0.88±0.03 & 0.40±0.04 & 0.06±0.02 & 0.85±0.04 & 0.36±0.05 & $<$ 0.05 \\
& PF+Roughening   & 0.89±0.03 & 0.48±0.04 & 0.10±0.03 & 0.86±0.04 & 0.44±0.05 & 0.08±0.03 \\
& PF+Rejuvenation & 0.89±0.03 & 0.52±0.04 & 0.16±0.03 & 0.86±0.04 & 0.48±0.05 & 0.12±0.03 \\

\midrule
\multirow{10}{*}{ADE$\downarrow$} 
& DEPF (ours) & \textbf{19±0.8} & \textbf{22±1.2} & \textbf{27±1.7} & \textbf{167±15} & \textbf{200±10} & \textbf{255±15} \\
\cmidrule{2-8}
& AGDC & 18±0.9 & 59±11 & timeout & 168±15 & 235±20 & timeout \\
& PC-DQN & 20±1.0 & 60±11 & timeout & 193±16 & 246±28 & timeout \\
& GMM-PFRL & 19±0.9 & 60±20 & timeout & 200±16 & 250±36 & timeout \\
\cmidrule{2-8}
& Infotaxis & 40±2.0 & 70±30 & timeout & timeout & timeout & timeout \\
& Entrotaxis & 50±2.5 & 75±25 & timeout & timeout & timeout & timeout \\
& DCEE & 45±2.2 & 55±3.5 & timeout & timeout & timeout & timeout \\
\cmidrule{2-8}
& PF+Jittering    & 22±1.0 & 65±8  & timeout & 180±15 & 250±22 & timeout \\
& PF+Roughening   & 21±0.9 & 55±6  & timeout & 175±15 & 235±20 & timeout \\
& PF+Rejuvenation & 21±0.9 & 50±5  & 90±12  & 170±14 & 225±18 & 285±35 \\

\midrule
\multirow{10}{*}{REV$\downarrow$} 
& DEPF (ours) & \textbf{0.10±0.05} & \textbf{0.10±0.05} & \textbf{0.10±0.05} & \textbf{0.10±0.05} & \textbf{0.10±0.05} & \textbf{0.10±0.05} \\
\cmidrule{2-8}
& AGDC & 0.10±0.05 & 0.10±0.05 & 0.10±0.05 & 0.12±0.05 & 0.40±0.15 & 1.50±0.30 \\
& PC-DQN & 0.12±0.05 & 0.12±0.05 & 0.12±0.05 & 0.15±0.07 & 0.45±0.17 & 1.50±0.30 \\
& GMM-PFRL & 0.11±0.04 & 0.11±0.04 & 0.11±0.04 & 0.13±0.06 & 0.42±0.16 & 1.50±0.30 \\
\cmidrule{2-8}
& Infotaxis & 1.30±0.06 & 1.30±0.06 & 1.30±0.06 & 1.80±0.08 & 3.00±0.15 & 4.00±0.20 \\
& Entrotaxis & 1.30±0.05 & 1.30±0.05 & 1.30±0.05 & 1.60±0.08 & 2.80±0.14 & 4.50±0.22 \\
& DCEE & 1.30±0.05 & 1.30±0.05 & 1.30±0.05 & 1.50±0.07 & 2.50±0.12 & 4.20±0.21 \\
\cmidrule{2-8}
& PF+Jittering    & 0.11±0.05 & 0.30±0.12 & 1.80±0.35 & 0.12±0.05 & 0.90±0.25 & 2.80±0.45 \\
& PF+Roughening   & 0.11±0.05 & 0.22±0.10 & 1.60±0.30 & 0.12±0.05 & 0.70±0.20 & 2.40±0.40 \\
& PF+Rejuvenation & 0.11±0.05 & 0.20±0.08 & 1.20±0.25 & 0.12±0.05 & 0.60±0.18 & 2.00±0.35 \\

\midrule
\multirow{10}{*}{LPS$\downarrow$} 
& DEPF (ours) & \textbf{0.20±0.01} & \textbf{0.20±0.01} & \textbf{0.20±0.02} & \textbf{0.20±0.01} & \textbf{0.20±0.01} & \textbf{0.20±0.01} \\
\cmidrule{2-8}
& AGDC & 0.20±0.01 & 2.60±0.03 & 12.50±0.15 & 0.20±0.01 & 2.60±0.03 & 12.60±0.15 \\
& PC-DQN & 0.23±0.02 & 2.64±0.04 & 12.60±0.18 & 0.23±0.02 & 2.65±0.04 & 12.70±0.19 \\
& GMM-PFRL & 0.25±0.01 & 2.65±0.03 & 12.55±0.16 & 0.25±0.01 & 2.68±0.03 & 12.57±0.16 \\
\cmidrule{2-8}
& Infotaxis & 0.60±0.03 & 3.30±0.03 & 12.50±0.20 & 0.62±0.03 & 3.32±0.03 & 12.70±0.20 \\
& Entrotaxis & 0.70±0.03 & 3.40±0.03 & 13.00±0.23 & 0.72±0.03 & 3.42±0.03 & 13.30±0.23 \\
& DCEE & 0.65±0.03 & 3.55±0.04 & 12.80±0.22 & 0.68±0.03 & 3.58±0.05 & 13.10±0.23 \\
\cmidrule{2-8}
& PF+Jittering    & 0.26±0.02 & 3.20±0.12 & 11.8±0.4 & 0.28±0.02 & 3.40±0.15 & 12.0±0.5 \\
& PF+Roughening   & 0.24±0.02 & 2.90±0.10 & 10.5±0.4 & 0.26±0.02 & 3.10±0.12 & 11.0±0.5 \\
& PF+Rejuvenation & 0.23±0.02 & 2.70±0.09 &  9.0±0.3 & 0.24±0.02 & 2.90±0.10 & 10.0±0.4 \\

\bottomrule
\end{tabular}%
}
\end{table}

\subsection{Performance Analysis Across Error Conditions and Environment Scales}

\begin{figure}[!t]
  \centering
  \includegraphics[width=\linewidth]{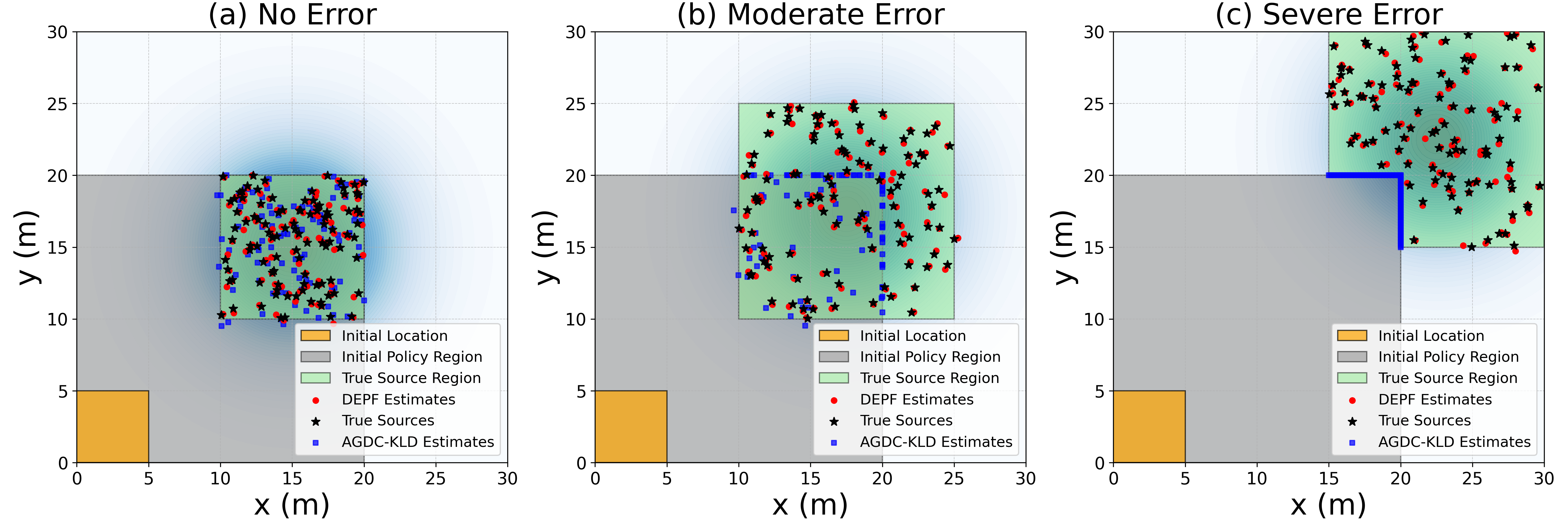}
  \caption{Visualization of Policy Error Scenarios in Emergency Response.}
  \label{fig:policy_estimates}
\end{figure}

We conduct a detailed evaluation of DEPF and nine baseline algorithms under varying levels of initial policy error—\textit{No Error}, \textit{Moderate Error}, and \textit{Severe Error}—each tested in both \textit{small-scale} (agent-domain ratio 1:30) and \textit{large-scale} (1:300) environments. Performance is assessed using four metrics: OCE, ADE, LPS, and REV. Quantitative results are presented in Table~\ref{tab:policy_error_experiments}, with spatial illustrations provided in Figure~\ref{fig:policy_estimates}, which visualizes the posterior estimates from 100 test trials of DEPF and AGDC under each policy error condition.

\textbf{No Error (Ideal).} When the initial prior support fully covers the true source region, all methods perform well. DEPF attains \textbf{OCE = 0.90}, \textbf{LPS = 0.20}, \textbf{ADE = 19} (small-scale) and \textbf{ADE = 167} (large-scale), with \textbf{REV = 0.10} in both scales. AGDC closely matches DEPF on all metrics (\textbf{OCE = 0.90/0.87}, \textbf{LPS = 0.20}, \textbf{ADE = 18/168}). RL baselines PC-DQN and GMM-PFRL also converge (OCE $\approx$ 0.80), while information-theoretic planners show varying success (e.g., Infotaxis achieves OCE $0.85$ in small-scale but degrades in large-scale). The three classical SMC perturbation baselines (PF+Jittering/Roughening/Rejuvenation) behave similarly or slightly worse than DEPF due to always-on diffusion incurring mild over-exploration cost in ideal conditions. \textbf{Moderate Error (Partial misalignment).} With partial overlap between prior support and ground truth, DEPF maintains high performance: \textbf{OCE = 0.90}, \textbf{LPS = 0.20}, \textbf{ADE = 22} (small) and \textbf{ADE = 200}, \textbf{REV = 0.10} (large). AGDC degrades substantially (\textbf{OCE = 0.45/0.42}, \textbf{LPS = 2.60}, \textbf{ADE = 59/235}) with increased timeouts at larger scale. PC-DQN and GMM-PFRL show similar or worse drops (OCE $\approx$ 0.40). Among planners, Infotaxis/Entrotaxis/DCEE underperform markedly (see Table). Notably, the classical PF perturbation baselines recover \emph{part} of the gap relative to RL/planning baselines: PF+Rejuvenation $>$ PF+Roughening $>$ PF+Jittering in OCE/LPS, consistent with resample--move providing the strongest rejuvenation. However, all three remain \emph{well below} DEPF in both accuracy and efficiency, especially in the large-scale setting where path lengths and timeouts increase. \textbf{Severe Error (Complete misalignment; S‑PSI test).} This scenario creates a strict zero-prior barrier: the initial prior support is disjoint from the true source region. DEPF remains stable with \textbf{OCE = 0.89} (small) and \textbf{0.88} (large), \textbf{LPS = 0.20}, and \textbf{ADE = 27/255}—all below the 100/300 step timeout thresholds and with \textbf{REV = 0.10}. In contrast, \emph{bootstrap PF under the S‑PSI baseline} (zero transition, no rejuvenation) cannot escape the prior support; accordingly, AGDC, PC-DQN, GMM-PFRL, and information-theoretic planners collapse (\textbf{OCE} $< 0.05$, \textbf{LPS} $>12.5$; frequent timeouts). The classical SMC perturbation baselines alleviate this limitation to a degree: PF+Rejuvenation achieves the highest non-zero success among the three (occasionally localizing in the small-scale severe case), followed by PF+Roughening and then PF+Jittering; nevertheless, their success rates remain low and many runs still timeout—especially in the large-scale severe case—highlighting the cost of always-on diffusion and the absence of belief-triggered control. \textbf{Scaling with domain size.} Across error conditions, most baselines suffer pronounced degradation when moving from 1:30 to 1:300: OCE drops, ADE/REV increase, and timeouts spike, reflecting insufficient long-range exploration when prior guidance is misleading. DEPF sustains \textbf{OCE = 0.88--0.90} with \textbf{LPS = 0.20} and \textbf{REV = 0.10} across scales, indicating robust scalability. Its advantage is clearest under Moderate/Severe misalignment, where belief-triggered support expansion and covariance-aligned diffusion provide minimal-bias exploration precisely when needed, in contrast to always-on perturbations. \textbf{Takeaways.} (i) In ideal settings, DEPF is competitive with the best baselines. (ii) Under partial misalignment, DEPF preserves accuracy/efficiency while RL/planning baselines and classical PF perturbations degrade, with PF+Rejuvenation the strongest among the latter yet still below DEPF. (iii) Under complete misalignment (S‑PSI test), DEPF consistently sustains high OCE and low LPS at both scales; all other methods either fail or achieve only occasional recovery (PF+Rejuvenation) with substantial timeouts. These results substantiate DEPF as a principled, data-triggered fallback when early-stage priors are unreliable.

\subsection{Ablation and Sensitivity Studies}
\label{sec:ablation}

We perform a series of ablations and sensitivity studies to understand how different components and hyperparameters of DEPF affect performance under the \textit{Severe Error} setting, where the true source lies entirely outside the initial prior support.

\begin{figure}[htbp]
  \centering
  \includegraphics[width=\linewidth]{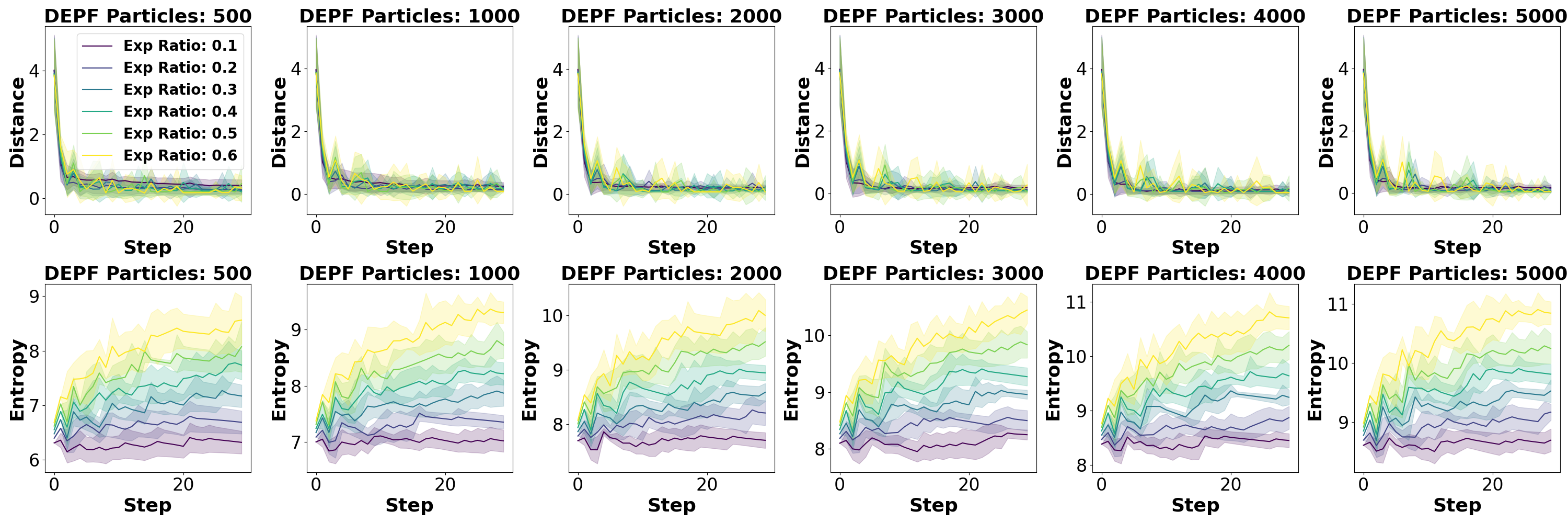}
  \caption{Impact of the number of particles and $\delta$ on DEPF performance under \textit{Severe Error}.}
  \label{fig:Impact_DEPF}
\end{figure}

\textbf{Sensitivity to particle number and margin parameter $\delta$.} We first analyze how DEPF’s performance is affected by the number of particles and the support margin ratio $\delta$, which determines the proportion by which the posterior support region is expanded beyond the initial prior boundary at each step. As shown in Figure~\ref{fig:Impact_DEPF}, increasing $\delta$ generally allows the model to converge more quickly—reflected in lower belief-to-goal distances—because particles are better able to explore outside the constrained prior region. However, this broader support also leads to higher posterior entropy, as particles become more dispersed and the belief distribution less concentrated. This trade-off is particularly pronounced at low particle counts (e.g., $N\!=\!500$–2000), where large $\delta$ causes belief imprecision. When using larger particle sets ($N \geq 3000$), DEPF maintains both fast convergence and acceptable entropy levels even at higher $\delta$, indicating that sufficient particle resolution can stabilize exploration-induced uncertainty. Overall, a moderate support ratio ($\delta \in [0.2, 0.4]$) paired with adequate particle numbers achieves a good balance between exploration and localization precision. These results reinforce DEPF’s ability to recover from severely misaligned initial state estimates by dynamically expanding support and refining the belief.

\begin{table}[hbtp]
\centering
\caption{Component-wise ablation of DEPF under \textit{No Error} and \textit{Severe Error}.}
\label{tab:ablation-components}
\resizebox{\textwidth}{!}{
\begin{tabular}{ccccccccccccc}
\toprule
\multirow{2}{*}{Idx} & \multicolumn{3}{c}{Components} & 
\multicolumn{4}{c}{No Error} & \multicolumn{4}{c}{Severe Error} \\
\cmidrule(lr){2-4}\cmidrule(lr){5-8}\cmidrule(lr){9-12}
& Entropy & Stoch.\ Diffusion & MH Valid. &
OCE & ADE & REV & LPS &
OCE & ADE & REV & LPS \\
\midrule
1 & $\checkmark$ & $\checkmark$ & $\checkmark$ &
$0.90\!\pm\!0.03$ & $19\!\pm\!0.8$ & $0.10\!\pm\!0.05$ & $0.20\!\pm\!0.01$ &
$0.88\!\pm\!0.03$ & $27\!\pm\!1.7$ & $0.10\!\pm\!0.05$ & $0.20\!\pm\!0.01$ \\
2 & $\times$ & $\checkmark$ & $\checkmark$ &
$0.90\!\pm\!0.03$ & $19\!\pm\!0.8$ & $0.10\!\pm\!0.05$ & $0.20\!\pm\!0.01$ &
$0.62\!\pm\!0.17$ & $34\!\pm\!2.4$ & $0.10\!\pm\!0.05$ & $3.00\!\pm\!0.45$ \\
3 & $\checkmark$ & $\times$ & $\checkmark$ &
$0.90\!\pm\!0.03$ & $19\!\pm\!0.8$ & $0.10\!\pm\!0.05$ & $0.20\!\pm\!0.01$ &
$0.00$ & \texttt{timeout} & $0.10\!\pm\!0.05$ & $12.50\!\pm\!0.15$ \\
4 & $\checkmark$ & $\checkmark$ & $\times$ &
\multicolumn{4}{c}{\texttt{not work}} &
\multicolumn{4}{c}{\texttt{not work}} \\
\bottomrule
\end{tabular}}
\end{table}

\textbf{Component-wise ablation.} We next disentangle the contribution of each DEPF component. The component-wise ablation in Table~\ref{tab:ablation-components} shows that all three mechanisms—entropy-based weight regularization, covariance-scaled stochastic diffusion, and MH-based acceptance—are jointly necessary for robust recovery. Removing stochastic diffusion collapses performance under \textit{Severe Error}, yielding timeouts and large residual errors despite correct behavior in the \textit{No Error} setting. Disabling entropy regularization substantially lowers OCE and increases LPS (e.g., $0.62$ OCE and $\sim\!3.0$ LPS), indicating premature weight concentration and insufficient posterior spread. Absent MH validation, the algorithm becomes ill-posed (\texttt{not work}), confirming the role of acceptance control in maintaining Bayesian consistency. These outcomes highlight that DEPF’s gains do not stem from ad-hoc noise, but from a calibrated, evidence-triggered expansion pipeline.

\textbf{Hyperparameter sensitivity: $\lambda$ and $A$.} We then study the kernel stability parameter $\lambda$, which controls the diagonal regularization added to the empirical covariance, and the KDE bandwidth constant $A$. The two tables below summarize the effect of these hyperparameters.

\begin{table}[hbtp]
\centering
\begin{minipage}{0.48\textwidth}
    \centering
    \caption{Sensitivity to $\lambda$.}
    \label{tab:ablation-lambda}
    \resizebox{\linewidth}{!}{%
    \begin{tabular}{ccccc}
    \toprule
    $\lambda$ & OCE$\uparrow$ & ADE$\downarrow$ & REV$\downarrow$ & LPS$\downarrow$ \\
    \midrule
    $1\!\times\!10^{-1}$ & $0.87\!\pm\!0.04$ & $22\!\pm\!1.5$ & $0.25\!\pm\!0.04$ & $0.10\!\pm\!0.05$ \\
    $1\!\times\!10^{-2}$ & $0.90\!\pm\!0.03$ & $19\!\pm\!0.8$ & $0.22\!\pm\!0.04$ & $0.10\!\pm\!0.05$ \\
    $1\!\times\!10^{-3}$ & $0.91\!\pm\!0.05$ & $17\!\pm\!1.2$ & $0.23\!\pm\!0.04$ & $0.10\!\pm\!0.05$ \\
    $1\!\times\!10^{-4}$ & $0.89\!\pm\!0.08$ & $18\!\pm\!2.1$ & $0.25\!\pm\!0.04$ & $0.10\!\pm\!0.05$ \\
    $1\!\times\!10^{-5}$ & $0.85\!\pm\!0.12$ & $25\!\pm\!3.5$ & $0.26\!\pm\!0.04$ & $0.10\!\pm\!0.05$ \\
    \bottomrule
    \end{tabular}%
    }
\end{minipage}
\hfill
\begin{minipage}{0.48\textwidth}
    \centering
    \caption{Sensitivity to KDE $A$.}
    \label{tab:ablation-bandwidth}
    \resizebox{\linewidth}{!}{%
    \begin{tabular}{ccccc}
    \toprule
    $A$ & Belief Dist.$\downarrow$ & Entropy$\downarrow$ & LPS$\downarrow$ & Steps$\downarrow$ \\
    \midrule
    0.10 & 2.10 & 0.37 & 0.23 & 25 \\
    0.30 & 1.00 & 0.26 & 0.21 & 23 \\
    0.50 & \textbf{0.22} & \textbf{0.17} & \textbf{0.20} & \textbf{18} \\
    1.00 & 0.35 & 0.22 & 0.22 & 20 \\
    2.00 & 1.10 & 0.34 & 0.26 & 27 \\
    \bottomrule
    \end{tabular}%
    }
\end{minipage}
\end{table}

Mid-range values of $\lambda$ ($10^{-3}$–$10^{-2}$) strike the best balance between numerical stability and diffusion accuracy, yielding the lowest ADE and competitive OCE. Too large a $\lambda$ overdamps useful moves and slows convergence; too small a $\lambda$ induces numerical instability and degraded metrics. For the bandwidth constant $A$, an intermediate value $A\!\approx\!0.5$ minimizes belief-to-goal distance, posterior entropy, and convergence steps, whereas either under- or over-diffusion (very small or very large $A$) harms recovery. This confirms the need for bandwidth matching to the evolving posterior covariance, as encoded by DEPF’s covariance-scaled perturbations.

\textbf{Hyperparameter sensitivity: $\beta$ and exploratory ratio.}
Finally, we study the entropy regularization strength $\beta$ and the ratio of exploratory particles per step, summarized in Table~\ref{tab:ablation-beta} and Table~\ref{tab:ablation-explore}.

\begin{table}[hbtp]
\centering
\begin{minipage}{0.48\textwidth}
    \centering
    \caption{Sensitivity to $\beta$.}
    \label{tab:ablation-beta}
    \resizebox{\linewidth}{!}{%
    \begin{tabular}{lcccc}
        \toprule
        $\beta$ & OCE$\uparrow$ & ADE$\downarrow$ & REV$\downarrow$ & LPS$\downarrow$ \\
        \midrule
        0.1 & $0.88\pm0.04$ & $22\pm1.2$ & $0.10\pm0.05$ & $0.25\pm0.03$ \\
        0.2 & $0.89\pm0.03$ & $20\pm1.0$ & $0.10\pm0.05$ & $0.22\pm0.02$ \\
        0.3 & $0.90\pm0.03$ & $19\pm0.8$ & $0.10\pm0.05$ & $0.20\pm0.01$ \\
        0.4 & $0.91\pm0.02$ & $18\pm0.7$ & $0.10\pm0.05$ & $0.19\pm0.01$ \\
        0.5 & $0.91\pm0.02$ & $18\pm0.6$ & $0.10\pm0.05$ & $0.20\pm0.01$ \\
        \bottomrule
    \end{tabular}}
\end{minipage}
\hfill
\begin{minipage}{0.48\textwidth}
    \centering
    \caption{Sensitivity to the ratio of exploratory.}
    \label{tab:ablation-explore}
    \resizebox{\linewidth}{!}{%
    \begin{tabular}{ccccc}
        \toprule
        Ratio (\%) & Belief Dist.$\downarrow$ & Entropy$\downarrow$ & LPS$\downarrow$ & Steps$\downarrow$ \\
        \midrule
        1\%  & 1.60 & 0.29 & 0.22 & 26 \\
        2\%  & 0.65 & 0.21 & 0.20 & 20 \\
        5\%  & \textbf{0.22} & \textbf{0.17} & \textbf{0.20} & \textbf{18} \\
        10\% & 0.30 & 0.19 & 0.21 & 19 \\
        20\% & 0.75 & 0.25 & 0.22 & 22 \\
        \bottomrule
    \end{tabular}}
\end{minipage}
\end{table}

As $\beta$ increases from $0.1$ to $0.4$, the task OCE steadily improves from $0.88\!\pm\!0.04$ to $0.91\!\pm\!0.02$, while the path cost (ADE) decreases from $22\!\pm\!1.2$ to $18\!\pm\!0.7$ and the localization error (LPS) drops from $0.25\!\pm\!0.03$ to $0.19\!\pm\!0.01$. Execution latency (REV) remains around $0.10\!\pm\!0.05$, indicating that entropy regularization primarily enhances convergence efficiency and localization accuracy without introducing additional delays. When $\beta$ is further increased to $0.5$, all metrics remain essentially unchanged, suggesting that performance enters a plateau and becomes robust within the range $\beta \in [0.3, 0.5]$. For the exploratory-particle ratio, there is a clear interior optimum around $5\%$ (Table~\ref{tab:ablation-explore}). Too few exploratory particles slow support expansion and delay recovery; too many inject excess randomness, raising entropy and slightly reducing stability. Together with the diffusion-related ablations, these results yield a practical recipe: pair a moderate exploratory ratio (about $5\%$) with a mid-range bandwidth ($A\!\approx\!0.5$), moderate smoothing ($\beta\!\approx\!0.3$–$0.4$), and a stability parameter $\lambda$ in $[10^{-3}, 10^{-2}]$ to reliably break S-PSI while maintaining efficient convergence.

\section{Conclusion}

In this work, we identified and formalized S-PSI, a baseline-specific limitation of bootstrap particle filters that arises under zero transition and no rejuvenation. To address this issue, we proposed the DEPF framework, which dynamically expands posterior support through exploratory particle injection, entropy-driven diffusion, and kernel-based perturbations validated by Metropolis–Hastings. Experiments on hazardous-gas localization tasks demonstrated that DEPF consistently corrects severe prior misalignment, substantially outperforming both classical perturbation strategies and RL-based baselines. By resolving bootstrap-specific lock-in while maintaining Bayesian rigor, DEPF offers a robust and practical solution for decision-support systems in emergency management.

\clearpage

\bibliography{iclr2026_conference}
\bibliographystyle{iclr2026_conference}

\clearpage

\addtocontents{toc}{\protect\setcounter{tocdepth}{2}}
\appendix
\renewcommand{\contentsname}{Appendix} 
\tableofcontents 
\newpage


\section{Particle Filtering: Recursive Approximation of Posterior Distributions}

The particle filtering process recursively approximates the posterior distribution \( P(\Theta_k \mid z_{1:k}) \) through four main steps: sampling, weight update, normalization, and resampling. These steps iteratively adapt particle sets to new observations, enabling sequential Bayesian inference.

\textbf{Step 1. Sampling (Prediction):} \cite{gordon1993novel} Particles are sampled from an importance distribution \( q(\Theta_k \mid \Theta_{k-1}, z_k) \), typically approximated by the state transition model \( P(\Theta_k \mid \Theta_{k-1}) \). The propagation of particles is expressed as:
\[
\Theta_k^{(i)} \sim P(\Theta_k \mid \Theta_{k-1}^{(i)}),
\]
where \( \Theta_k^{(i)} \) denotes the \( i \)-th particle at time \( k \).

\textbf{Step 2. Weight Update:} \cite{doucet2001introduction} Particle weights are updated to reflect their likelihood given the new observation. The unnormalized weights are calculated as:
\[
\tilde{w}_k^{(i)} = w_{k-1}^{(i)} \cdot \frac{P(z_k \mid \Theta_k^{(i)}) P(\Theta_k^{(i)} \mid \Theta_{k-1}^{(i)})}{q(\Theta_k^{(i)} \mid \Theta_{k-1}^{(i)}, z_k)},
\]
where \( P(z_k \mid \Theta_k^{(i)}) \) is the observation likelihood, and \( q(\Theta_k \mid \Theta_{k-1}, z_k) \) is the importance distribution used in sampling.

\textbf{Step 3. Normalization of Weights:} \cite{liu1998sequential} Weights are normalized to maintain their probabilistic interpretation:
\[
w_k^{(i)} = \frac{\tilde{w}_k^{(i)}}{\sum_{j=1}^N \tilde{w}_k^{(j)}}.
\]

\textbf{Step 4. Resampling:} \cite{doucet2000sequential} To mitigate weight degeneracy (domination by a few particles), resampling occurs when the effective particle number \( N_{\text{eff}} \) falls below a threshold:
\[
N_{\text{eff}} = \frac{1}{\sum_{i=1}^N (w_k^{(i)})^2}.
\]
Particles are then resampled based on current weights, and weights are reset uniformly:
\[
w_k^{(i)} = \frac{1}{N}.
\]

These recursive steps approximate the posterior \( P(\Theta_k \mid z_{1:k}) \). However, under the \emph{stationary bootstrap baseline with zero transition and no rejuvenation}, particle trajectories remain confined to the initial prior support. We formalize this diagnostic effect as \textbf{Stationarity-Induced Posterior Support Invariance (S-PSI)}: the posterior cannot expand beyond the prior support in this baseline setting. Importantly, S-PSI is \emph{not} an inherent property of particle filtering—classical perturbation strategies such as jittering, roughening, or resample--move can, in principle, relax this constraint, and are included as baselines in our experiments.

\section{Stationarity-Induced Posterior Support Invariance (S-PSI)}
\label{sec:PSI}

Let the prior \(p_{0}(\Theta)\) have support
\(\mathcal{S}_{0}:=\operatorname{supp}p_{0}(\Theta)
     =\{\theta\mid p_{0}(\theta)>0\}\subset\mathbb{R}^{7}\).
If particles are initialized exclusively within \( \mathcal{S}_0 \), then under a \emph{stationary bootstrap baseline with zero transition and no rejuvenation}, the posterior distributions at all subsequent time steps will satisfy:
\[
\text{supp}(p(\Theta \mid z_{1:k})) \subseteq \mathcal{S}_0, \quad \forall k.
\]

We refer to this diagnostic baseline effect as \textbf{Stationarity-Induced Posterior Support Invariance (S-PSI)}. It indicates that:

\begin{quote}
Even if observations strongly suggest a source location outside the prior support, the particle filter cannot estimate it—\textbf{not due to insufficient likelihood}, but because no particles exist in that region under the baseline assumptions.
\end{quote}

Consequently, estimation fails if the true source \(\Theta^*\) lies outside the initial prior support:
\[
\Theta^* \notin \mathcal{S}_0 \quad \Rightarrow \quad p(\Theta^* \mid z_{1:k}) = 0, \quad \forall k.
\]

\begin{proposition}[S-PSI under the stationary bootstrap baseline]
Given particles initialized exclusively within the prior support \(\mathcal{S}_0\), the particle support range \(\mathcal{S}_k\) at any time \(k\) satisfies:
\[
\mathcal{S}_k \subseteq \mathcal{S}_0, \quad \forall k \geq 0,
\]
with the base case:
\[
\mathcal{S}_0 = \mathcal{S}_{\text{prior}}.
\]
Thus, under the baseline assumptions, the posterior support remains invariant over time:
\[
\mathcal{S}_k \subseteq \mathcal{S}_{\text{prior}}, \quad \forall k \geq 0.
\]
\end{proposition}

\begin{proof}
By definition, the particle support \(\mathcal{S}_k\) at step \(k\) is:
\[
\mathcal{S}_k = \bigcup_{i=1}^N \{\Theta_k \in \mathbb{R}^7 : P(\Theta_k \mid \Theta_{k-1}^{(i)}) > 0\}, \quad \Theta_{k-1}^{(i)} \in \mathcal{S}_{k-1}.
\]

For the base case at \(k = 0\), by construction:
\[
\mathcal{S}_0 = \mathcal{S}_{\text{prior}} = \{\Theta_0 \in \mathbb{R}^7 : p_0(\Theta_0) > 0\}.
\]

Assume inductively that \(\mathcal{S}_{k-1} \subseteq \mathcal{S}_{\text{prior}}\). Particle propagation at time \(k\) depends entirely on particles from step \(k-1\), and under the stationary baseline assumption (no transition noise, no rejuvenation), we have:
\[
\mathcal{S}_k \subseteq \mathcal{S}_{k-1}.
\]

Combining this with the inductive hypothesis:
\[
\mathcal{S}_k \subseteq \mathcal{S}_{\text{prior}}, \quad \forall k \geq 0.
\]

Thus, particles cannot move beyond the initial support boundary under this baseline, making the posterior support invariant over time.
\end{proof}

\paragraph{Remarks.}
S-PSI is \emph{not} an inherent property of particle filtering but a consequence of the stationary bootstrap baseline. Classical perturbation strategies such as jittering, roughening, or resample--move rejuvenation can, in principle, relax this invariance and are included as baselines in \S \ref{sec:Experiment}. Our proposed method DEPF builds upon this insight by introducing belief-triggered exploratory injection, covariance-scaled diffusion, and MH validation, enabling adaptive support expansion only when data contradict the current belief.

\newpage
\section{Theoretical Analysis and Justification}
\label{appendix:theory_analysis}

\subsection{Introduction}
In this appendix, we provide a theoretical analysis of the Diffusion-Enhanced Particle Filtering (DEPF) framework proposed in Section~4, focusing particularly on the rationality, correctness, and asymptotic unbiasedness of its mechanisms: exploratory particle injection, entropy-based weight smoothing, kernel-induced stochastic perturbation, and Metropolis--Hastings (MH) validation.

\subsection{Exploratory Particles and Modified Prior}

By injecting exploratory particles uniformly drawn from an expanded region $B_k$ beyond the initial prior support $S_0$, DEPF implicitly modifies the initial prior distribution $p_0(\Theta)$ as:
\begin{equation}
p_{\text{new-prior}}(\Theta) = (1 - \epsilon) p_{0}(\Theta) + \epsilon U(B_k),\quad \epsilon \ll 1,
\end{equation}
where $U(B_k)$ denotes the uniform distribution on $B_k$. This is equivalent to relaxing the diagnostic \textbf{Stationarity-Induced Posterior Support Invariance (S-PSI)} baseline, which arises under a stationary bootstrap PF with zero transition and no rejuvenation (\S \ref{subsec:PSI}). Since $\epsilon$ is small, this injection introduces minimal bias while enabling posterior support expansion outside $S_0$.

\subsection{Entropy-Based Weight Smoothing}

To preserve particle diversity, DEPF smooths weights toward higher entropy. One implementation is:
\begin{equation}
w_k^{(i)} \leftarrow w_k^{(i)} + \beta H(w_k),\quad H(w_k)=-\sum_{i=1}^N w_k^{(i)}\log(w_k^{(i)}+\epsilon),
\end{equation}
which can be viewed as Bayesian smoothing \citep{doucet2001sequential,liu1998sequential}. In practice, we also adopt a standard \emph{tempering} scheme
\[
\tilde w_k^{(i)} \propto (w_k^{(i)})^{1/T_k},\quad w_k^{(i)}=\tilde w_k^{(i)}/\sum_j \tilde w_k^{(j)},
\]
with $T_k\geq 1$ adapted from entropy gaps, ensuring theoretical coherence with standard SMC while mitigating degeneracy.

\subsection{Kernel-Induced Stochastic Diffusion}

Particles are further perturbed by a Gaussian kernel:
\begin{equation}
\Delta\Theta_k^{(i)} \sim h_{\text{opt}} \, L \, \mathcal{N}(0, I),\quad
h_{\text{opt}} = A \cdot N^{-\tfrac{1}{n+4}},\quad
\Theta_k^{(i)} \leftarrow \Theta_k^{(i)} + \Delta\Theta_k^{(i)},
\end{equation}
where $LL^\top=\Sigma$ is the Cholesky factorization of the weighted covariance. This is equivalent to a kernel density approximation \citep{silverman2018density}, and ensures that as $N\to\infty$ the perturbed empirical measure converges to the true posterior.

\subsection{Metropolis--Hastings Validation}

To maintain Bayesian coherence, proposals are validated by a standard MH acceptance step:
\begin{equation}
\alpha_i = \min\!\left(1,\frac{\pi(\Theta_i')\,q(\Theta_i \mid \Theta_i')}{\pi(\Theta_i)\,q(\Theta_i' \mid \Theta_i)}\right),
\end{equation}
with $\pi(\Theta)\propto p(z_k\mid \Theta)\,p(\Theta \mid z_{1:k-1})$. For symmetric Gaussian $q$, this reduces to
\begin{equation}
\alpha_i = \min\!\left(1, \frac{p(z_k\mid \Theta_i')}{p(z_k\mid \Theta_i)}\right).
\end{equation}
This guarantees detailed balance and convergence to the true posterior distribution \citep{hastings1970monte}.

\subsection{Convergence and Theoretical Guarantees}

Combining these mechanisms, DEPF ensures:
\begin{enumerate}
    \item Minimal bias from exploratory injection, expanding support beyond $S_0$ only when warranted.
    \item Stability of particle diversity due to entropy-based weight smoothing.
    \item Asymptotic convergence via kernel-induced diffusion validated by MH acceptance.
\end{enumerate}
Therefore, DEPF is asymptotically unbiased:
\begin{equation}
\lim_{N \rightarrow \infty} p_{\text{particle}}(\Theta\mid z_{1:k}) = p_{\text{true posterior}}(\Theta\mid z_{1:k}).
\end{equation}

\section{Proof of Support Range Expansion Beyond the Prior Boundary}

Under the stationary bootstrap baseline with zero transition and no rejuvenation, particle filters exhibit the diagnostic phenomenon we term \textbf{Stationarity-Induced Posterior Support Invariance (S-PSI)}: the posterior support remains confined to the prior support set \(\mathcal{S}_{\text{prior}}\). This is not an inherent limitation of PF but a consequence of the baseline assumptions (\S \ref{subsec:PSI}). Classical remedies such as jittering, roughening, or resample--move can, in principle, expand support, but they do so in an always-on and often inefficient manner. Here, we formally prove that with DEPF’s enhancements—exploratory injection, entropy-based regularisation, and kernel-induced stochastic diffusion validated by MH—the support set \(\mathcal{S}_k\) can expand beyond \(\mathcal{S}_{\text{prior}}\).

\begin{proposition}[Expansion of Support Range Beyond S-PSI]
With the proposed enhancements, the support range \( \mathcal{S}_k \) satisfies the recursion
\[
\mathcal{S}_k^{\text{new}} = \mathcal{S}_k \cup \mathcal{B}, 
\qquad 
\mathcal{S}_{k+1} = \mathcal{S}_k^{\text{new}} \oplus h_{\text{opt}},
\]
where \( \mathcal{B} \) is the extended bounding box sampled by exploratory particles and \( \oplus h_{\text{opt}} \) denotes kernel-induced diffusion. Starting from
\[
\mathcal{S}_0 = \mathcal{S}_{\text{prior}},
\]
for any target state \( \Theta^* \in \mathcal{B} \), there exists a finite step \( k \) such that
\[
\Theta^* \in \mathcal{S}_k.
\]
\end{proposition}

\begin{proof}
\textbf{Step 1: Base case.} At \(k=0\), \(\mathcal{S}_0=\mathcal{S}_{\text{prior}}\).

\textbf{Step 2: Exploratory injection.} At step \(k>0\), particles sampled from \(\mathcal{U}(\mathcal{B})\) enlarge the predictive support:
\[
\mathcal{S}_k^{\text{new}} = \mathcal{S}_k \cup \mathcal{B}.
\]

\textbf{Step 3: Weight-based survival.} Exploratory particles obtain weights proportional to their likelihood:
\[
w_k^{(j)} \propto p(z_k \mid \Theta_k^{(j)}).
\]
Particles consistent with observations survive resampling, ensuring that regions supported by data persist.

\textbf{Step 4: Kernel-induced diffusion.} Surviving particles are further perturbed:
\[
\Delta \Theta_k^{(i)} \sim h_{\text{opt}} L \mathcal{N}(0,I), 
\qquad
\Theta_k^{(i)} \leftarrow \Theta_k^{(i)} + \Delta \Theta_k^{(i)},
\]
which dilates the support: \(\mathcal{S}_{k+1}=\mathcal{S}_k^{\text{new}} \oplus h_{\text{opt}}\).

\textbf{Step 5: Induction.} Iterating steps 2–4, if \(\Theta^*\in \mathcal{B}\), then with positive probability it survives weighting and resampling, and by induction there exists finite \(k\) such that \(\Theta^*\in \mathcal{S}_k\).
\end{proof}

\paragraph{Remarks.} This result shows that DEPF breaks the S-PSI constraint by combining: (i) exploratory injection to seed new regions, (ii) likelihood-driven survival so only data-supported regions expand, and (iii) kernel diffusion to propagate local coverage. Together with MH validation (\S \ref{subsec:Problem_Setup}), these steps ensure support expansion is both data-triggered and statistically coherent.

\newpage
\section{Detailed Scenario Specifications}
\label{app:scenario-spec}

The three policy error scenarios are defined as follows: \textcolor{cyan}{\textbf{(1) No Error (Ideal Scenario)}} in Figure~\ref{fig:policy_error}(a) represents an ideal policy scenario in which the initial government decision on the disaster source location is completely accurate. Specifically, the initial particle distribution (gray region: $(0,20)\times(0,20)$) accurately covers the true source region (green region: $(10,20)\times(10,20)$). This scenario serves as a baseline, assessing the best-case performance without initial decision uncertainty. \textcolor{cyan}{\textbf{(2) Moderate Error (Partial Misalignment Scenario)}} in Figure~\ref{fig:policy_error}(b) simulates a realistic and moderate error scenario where the government’s initial policy assumptions about the disaster location (gray region: $(0,20)\times(0,20)$) are partially misaligned, overlapping only partially with the true disaster source region (green region: $(10,25)\times(10,25)$). This setup tests the ability of our method and baseline algorithms to adaptively correct moderate inaccuracies in initial policy assumptions, reflecting realistic emergency management conditions. \textcolor{cyan}{\textbf{(3) Severe Error (Complete Misalignment, PSI Scenario)}} in Figure~\ref{fig:policy_error}(c) represents the most challenging and realistic scenario, termed the Posterior Support Invariance scenario, in which the initial policy completely excludes the true disaster location. Here, the initial policy region (gray region: $(0,20)\times(0,20)$) and the true disaster source regions (green regions: $(15,30)\times(20,30) \cup (20,30)\times(15,20)$) are entirely disjoint, creating a strict zero-prior barrier. This scenario explicitly tests each algorithm’s robustness and ability to dynamically correct severe initial policy errors—situations where \textcolor{black}{bootstrap particle filtering methods typically fail} due to their inherent inability to extend beyond the initial support.

\newpage
\section{ISLC environments}
\label{sec:ISLC_env}

In this section, we describe the simulation environment adapted to a reinforcement learning and planning framework suitable for conducting research as outlined in Section \S \ref{sec:Preliminaries} of the paper.
\begin{table}[htbp]
    \centering
    \caption{Parameter Distributions for the Training Scenarios}
    \label{tab:training_parameters}
    
    \begin{tabular}{|c|c|}
        \hline
        \textbf{Source Parameter} & \textbf{Distribution} \\ \hline
        location of field source \( x_s\) & Uniform \( \mathcal{U}(5, 25) \) \\ \hline
        location of field source \( y_s \) & Uniform \( \mathcal{U}(5, 25) \) \\ \hline
        Release Strength \( q_s \) & Uniform \( \mathcal{U}(10, 3000)\) \\ \hline
        Wind Speed \( u_s\) & Uniform \( \mathcal{U}(0, 6) \) \\ \hline
         Decay parameter \(\lambda\) & Uniform \( \mathcal{U}(0, 8) \)\\ \hline
        Diffusivity \( d_s \) & Uniform \( \mathcal{U}(1, 5) \)\\ \hline
        Sensor Noise \( \epsilon \) & Fixed at 0.5 \\ \hline
        Environmental Noise \( \sigma \) & Fixed at 0.4 \\ \hline
        Effective Samples $N_{eff}$  & Fixed at 0.6 \\ \hline
    \end{tabular}
\end{table}
\subsection{Environment Overview}

The environment simulates a pollutant dispersion scenario within a defined two-dimensional spatial boundary. An RL agent operates within this environment by sequentially selecting positions to gather sensor measurements. Importantly, the environment itself does not explicitly provide reward signals. Instead, the agent has access only to its positional coordinates and the concentration measurements obtained at those positions.

The spatial boundaries for the simulation environment are as follows:
\begin{itemize}
\item \(x\)-axis: from \(0\) to \(25\) meters.
\item \(y\)-axis: from \(0\) to \(25\) meters.
\end{itemize}

\subsection{Pollutant Dispersion Model}

\begin{figure}[!t]
  \centering
  \begin{subfigure}[t]{0.48\linewidth}
    \includegraphics[width=\linewidth]{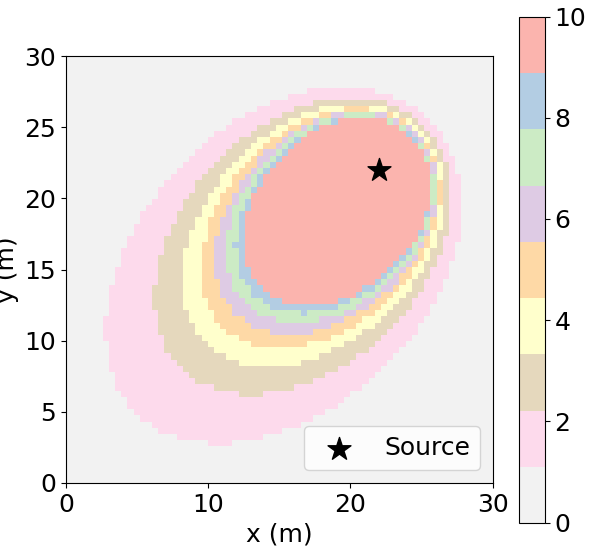}
    \caption{Plume Model}
    \label{fig:plume_model}
  \end{subfigure}
  \hfill
  \begin{subfigure}[t]{0.48\linewidth}
    \includegraphics[width=\linewidth]{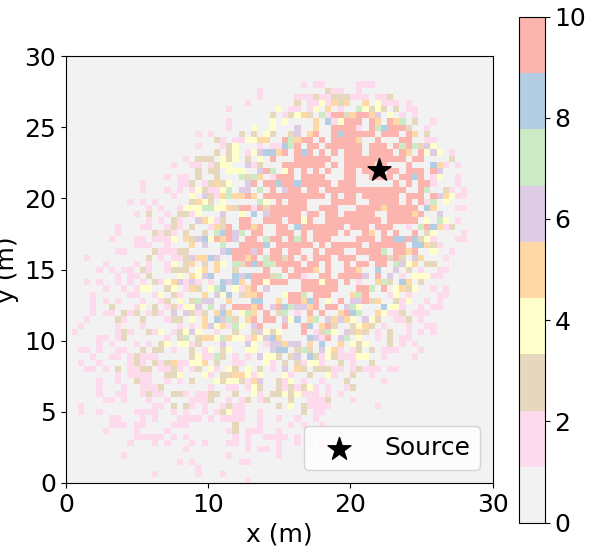}
    \caption{Plume Model with noise}
    \label{fig:plume_model_noise}
  \end{subfigure}
  \caption{Gaussian Plume Model Visualization.}
  \label{fig:GPM_visualization}
\end{figure}

Pollutant concentrations are simulated using a Gaussian plume dispersion model in Figure \ref{fig:GPM_visualization}, characterized by source emission parameters and environmental conditions:

\[
C = \frac{q_s}{4\pi d_s \cdot \text{dist}} \exp\left(-\frac{(x - x_{\text{source}})u\cos\phi + (y - y_{\text{source}})u\sin\phi + \text{dist}}{2d_s}\right),
\]

where \(q_s\) is the pollutant release rate, \(u\) is the wind speed, \(\phi\) is the wind direction, \(d_s\) represents dispersion coefficients, \(\text{dist}\) denotes the Euclidean distance between the agent's position and the pollutant source.

\subsection{Agent Observations and Actions}

At each time step, the RL agent observes:
\begin{itemize}
\item Its current two-dimensional coordinates \((x, y)\).
\item The pollutant concentration $z$ measured at that location.
\end{itemize}

Based on these observations, the agent selects its next position within the spatial boundary constraints. The agent’s action space consists of discrete movements within predefined distances, typically in increments of \(1\) meter in either the \(x\) or \(y\) direction.

\subsection{No Explicit Reward Structure}

Crucially, the environment itself does not provide any explicit reward signals. The agent must indirectly \textbf{infer} useful positional information solely from the sequence of measured pollutant concentrations. Consequently, the agent's learning process depends entirely on interpreting these indirect observations, making it a suitable scenario to investigate methods dealing with sparse or implicit rewards, as discussed extensively in Section \S \ref{sec:Preliminaries}.

\newpage
\section{Detailed Description of Evaluation Metrics}

We provide detailed definitions and interpretations for each evaluation metric used in the experiments, accompanied by illustrative examples:

\begin{itemize}

\item \textbf{Operational Completion Efficacy (OCE)}: This metric quantifies the proportion of emergency response missions successfully completed. For instance, if an algorithm successfully completes 90 out of 100 missions, it achieves an OCE of 0.90, indicating high reliability and effectiveness in operational scenarios.

\item \textbf{Average Deployment Efficiency (ADE)}: ADE measures the average total distance traveled by response units during their missions. For example, if algorithm A requires units to travel an average of 15 kilometers per mission, while algorithm B achieves the same objectives in only 10 kilometers, algorithm B demonstrates a better (lower) ADE, suggesting superior efficiency in routing and resource usage.

\item \textbf{Response Execution Velocity (REV)}: REV assesses the total time from initial deployment to mission completion. As an example, consider two algorithms where one completes missions in an average of 20 minutes, while the other requires 30 minutes on average. The first algorithm, with its shorter completion times (higher REV), is preferable for urgent, time-sensitive scenarios.

\item \textbf{Localization Precision Score (LPS)}: LPS evaluates localization accuracy by measuring the average spatial difference between estimated and actual source positions. For instance, if algorithm A consistently estimates the source within an average of 2 meters from its actual location, whereas algorithm B averages a 5-meter error, algorithm A demonstrates a superior (lower) LPS, reflecting more precise and reliable source localization.

\end{itemize}

These metrics collectively offer a comprehensive evaluation framework, enabling thorough assessments of both algorithmic efficiency and practical effectiveness in complex operational environments.

\section{Baseline Methods}
\label{sec:baselines}

Below, all baseline methods are described using the unified notation from this paper (\( \Theta\) for the hidden source vector, \(p_k\) for robot position, \(z_{1:k}\) for observations, \(b_k(\Theta)\) for the belief, \(a_k\) for actions, \(r_k\) for rewards, and discount factor \(\gamma\)).

\subsection{RL-PF Methods}

\paragraph{AGDC.}
A self-terminating RL framework that uses the particle-filter posterior’s uncertainty to decide when the source has been located, thus avoiding handcrafted sparse rewards.  
\textbf{Belief approximation:}  
\[
b_k(\Theta)\approx\sum_{i=1}^N w_k^{(i)}\,\delta(\Theta-\Theta_k^{(i)}),\quad\sum_i w_k^{(i)}=1.
\]  
\textbf{Stop criterion \& intrinsic reward:} Compute the posterior covariance \(C_k=\mathrm{Cov}_{b_k}(\Theta)\) and its diagonal standard deviation \(\mathrm{STD}_k=\sqrt{\mathrm{diag}(C_k)}\). If \(\|\mathrm{STD}_k\|<\zeta\), set  
\[
r_k=1\quad\text{and terminate;}\quad\text{otherwise }r_k=0.
\]  
\textbf{AGDC-KLD:} Uses the Kullback–Leibler divergence between successive beliefs as reward,  
\[
r_k = D_{\mathrm{KL}}\bigl(b_{k+1}\,\|\,b_k\bigr),
\]
driving actions that yield maximal information gain.

\paragraph{PC-DQN.}
The first source-search DRL method combining DQN with PF. It compresses particles into a compact 6-D feature vector using DBSCAN, then trains a vanilla DQN for discrete moves.  
\textbf{Reward:} \(-1\) per step until the source is declared (then \(+20\)).  

\paragraph{GMM-PFRL.}
A continuous-action RL approach using DDPG with GRU memory. The PF posterior is approximated by a Gaussian mixture model,  
\[
b_k(\Theta)\approx\sum_{m=1}^M \pi_m \,\mathcal{N}(\Theta\mid \mu_m,\Sigma_m),
\]  
with mixture parameters used as state input. This method handles continuous moves and outperforms information-theoretic baselines.

\subsection{Information-Theoretic Planners}

\paragraph{Infotaxis.}
A POMDP planner minimizing expected posterior variance:  
\[
a_k=\arg\min_a\,\mathbb{E}_{z_{k+1}}\bigl[\mathrm{Var}(\Theta\mid z_{1:k},z_{k+1})\bigr].
\]

\paragraph{Entrotaxis.}
A simplified Infotaxis variant using entropy of the predictive distribution as reward,  
\[
a_k=\arg\max_a\,H[b_{k+1}].
\]

\paragraph{DCEE.}
Dual-control exploration-exploitation planning:  
\[
a_k=\arg\min_a\Bigl\{\|\mu_{k+1}-p_{k+1}\|^2+\lambda\,\mathrm{tr}\bigl(\mathrm{Cov}[\Theta\mid z_{1:k+1}]\bigr)\Bigr\}.
\]

\subsection{Classical SMC Perturbation Baselines}

Under the \emph{stationary bootstrap baseline with zero transition and no rejuvenation}, PF exhibits \textbf{Stationarity-Induced Posterior Support Invariance (S-PSI)}: the posterior cannot escape the initial prior support (\S \ref{subsec:PSI}). This is not an inherent PF limitation; classical perturbation methods can relax S-PSI by injecting noise, albeit in an always-on and uncontrolled manner. We include three such baselines:

\paragraph{PF+Jittering.}
Each particle receives a small Gaussian perturbation after resampling:  
\[
\Theta_k^{(i)} = \Theta_{k-1}^{(i)} + \epsilon_k, \quad \epsilon_k \sim \mathcal{N}(0, \Sigma).
\]  
Ensures support expansion but without control over when/where to perturb.

\paragraph{PF+Roughening.}
Similar to jittering but with covariance scaled to $N$ and state dimension, encouraging broader diffusion:  
\[
\Theta_k^{(i)} = \Theta_{k-1}^{(i)} + \epsilon_k, \quad \epsilon_k \sim \mathcal{N}(0, \Sigma_{\text{rough}}).
\]  
More exploratory, but risks over-dispersion.

\paragraph{PF+Rejuvenation (Resample--Move).}
After resampling, particles undergo MH moves with Gaussian proposals:  
\[
\Theta' \sim \mathcal{N}(\Theta,\sigma_{\text{rm}}^2\Sigma),\qquad
\alpha = \min\!\left(1,\frac{p(z_k\mid \Theta')}{p(z_k\mid \Theta)}\right).
\]  
Provides stronger correction but increases cost.

\paragraph{Comparison to DEPF.}
While these baselines demonstrate that S-PSI is not universal, they lack principled triggers and validation. In contrast, DEPF expands support \emph{only when} observations contradict the belief, scales diffusion with posterior covariance, and validates proposals via MH acceptance, ensuring minimal-bias and Bayesian coherence.

\section{Diffusion-Driven Particle Filtering}
\label{app:diffusion-pf}

\subsection{Stationarity-Induced Posterior Support Invariance (S-PSI)}
Under the \emph{stationary bootstrap baseline with zero transition and no rejuvenation}, particle filters exhibit
\textbf{Stationarity-Induced Posterior Support Invariance (S-PSI)}: particle support remains confined to the initial prior region $S_{\mathrm{prior}}$. If the true state $\Theta^*$ lies outside this support, then
\[
  \mathrm{supp}\bigl(p(\Theta\mid z_{1:k})\bigr) \subseteq S_{\mathrm{prior}},\quad\forall k,
\]
so Bayesian updates cannot recover it. This baseline pathology motivates explicit support-expansion mechanisms. Classical perturbations (jittering, roughening, resample--move) can in principle relax S-PSI but do so in an always-on and uncontrolled manner. 

\subsection{Diffusion-Driven Support Expansion}
To overcome S-PSI in a principled way, DEPF augments filtering with four components:

\begin{enumerate}
  \item \textbf{Belief-Triggered Exploratory Particles:} 
  At each step, a small fraction $E$ of particles is redrawn from an extended box $B\supseteq S_{\mathrm{prior}}$:
  \[
    \Theta^{(j)}_k \sim \mathcal U(B),\quad j\in E,\quad w^{(j)}_k=\tfrac{\epsilon}{|E|},\ \epsilon\ll1.
  \]
  These particles seed new regions but survive only if supported by likelihood.

  \item \textbf{Entropy/Tempering Regularisation:} 
  To prevent weight collapse, we smooth weights toward higher entropy, e.g. via tempering:
  \[
    \tilde w_k^{(i)} \propto (w_k^{(i)})^{1/T_k},\quad 
    w_k^{(i)}=\tilde w_k^{(i)}/\sum_j \tilde w_k^{(j)},\quad T_k\ge1.
  \]

  \item \textbf{Covariance-Scaled Stochastic Diffusion:} 
  Surviving particles are perturbed by a Gaussian kernel aligned with the weighted covariance:
  \[
    \Delta\Theta^{(i)}_k \sim h_{\mathrm{opt}}\,L\,\mathcal N(0, I),
    \quad LL^T=\Sigma,
  \]
  with $\mu=\sum_i w^{(i)}_k\Theta^{(i)}_k$, 
  $\Sigma=\sum_i w^{(i)}_k(\Theta^{(i)}_k-\mu)(\Theta^{(i)}_k-\mu)^T+\lambda I$, 
  and $h_{\mathrm{opt}}=A N^{-1/(n+4)}$.

  \item \textbf{MH Validation:} 
  Each perturbed particle $\Theta'$ is accepted with standard Metropolis--Hastings probability
  \[
    \alpha_i = \min\!\left(1,\ \frac{p(z_k\mid \Theta')}{p(z_k\mid \Theta^{(i)}_k)}\right),
  \]
  since the Gaussian proposal is symmetric. Accepted proposals expand support consistently with Bayesian updates.
\end{enumerate}

These steps yield the recursion
\[
  S_{k+1}=\bigl(S_k\cup B\bigr)\oplus h_{\mathrm{opt}},
\]
allowing support expansion beyond $S_{\mathrm{prior}}$ only when justified by data.
\subsection{Pseudocode}
\begin{algorithm}[H]
\caption{Diffusion-Enhanced Particle Filter (DEPF)}
\label{alg:diffusion-pf}
\begin{algorithmic}[1]
\REQUIRE prior support $S_{\mathrm{prior}}$, extended box $B$, particle count $N$, thresholds $\eta,\epsilon$, smoothing parameter $\beta$, kernel parameters $A,\lambda$, horizon $K$
\STATE Initialize $\{(\Theta_0^{(i)},w_0^{(i)})\}_{i=1}^N$ with $\Theta_0^{(i)}\sim S_{\mathrm{prior}}$ and $w_0^{(i)}=1/N$
\FOR{$k=1$ \textbf{to} $K$}
  \STATE Select an exploratory index set $E\subseteq\{1,\dots,N\}$
  \FOR{each $i\in E$}
    \STATE $\Theta_k^{(i)} \gets \mathcal U(B)$;\quad $w_k^{(i)} \gets \epsilon/|E|$
  \ENDFOR
  \FOR{each $i\in\{1,\dots,N\}\setminus E$}
    \STATE $\Theta_k^{(i)} \gets \Theta_{k-1}^{(i)}$ \COMMENT{static prediction}
  \ENDFOR
  \STATE \textbf{Weight update:}\quad $w_k^{(i)} \propto w_{k-1}^{(i)}\,p(z_k\mid \Theta_k^{(i)})$ for $i=1,\dots,N$
  \STATE \textbf{Entropy/tempering smoothing (optional):}\quad $\tilde w_k^{(i)}\propto (w_k^{(i)})^{1/T_k}$; \quad $w_k^{(i)}\gets \tilde w_k^{(i)}/\sum_{j=1}^N \tilde w_k^{(j)}$
  \IF{$\mathrm{ESS}(\{w_k^{(i)}\}_{i=1}^N) < \eta$}
    \STATE Resample $\{\Theta_k^{(i)}\}_{i=1}^N$ according to $\{w_k^{(i)}\}_{i=1}^N$;\quad set $w_k^{(i)} \gets 1/N$
  \ENDIF
  \STATE $\mu \gets \sum_{i=1}^N w_k^{(i)} \Theta_k^{(i)}$;\quad $\Sigma \gets \sum_{i=1}^N w_k^{(i)}(\Theta_k^{(i)}-\mu)(\Theta_k^{(i)}-\mu)^{\top} + \lambda I$
  \STATE Compute $L$ via $\Sigma = LL^{\top}$;\quad $h_{\mathrm{opt}} \gets A\,N^{-1/(n+4)}$
  \FOR{$i=1$ \textbf{to} $N$}
    \STATE Sample $\delta\sim \mathcal N(0,I)$;\quad $\Delta \gets h_{\mathrm{opt}}\,L\,\delta$
    \STATE $\Theta' \gets \Theta_k^{(i)} + \Delta$
    \STATE $\alpha \gets \min\!\left(1,\ \dfrac{p(z_k\mid \Theta')}{p(z_k\mid \Theta_k^{(i)})}\right)$ \COMMENT{MH with symmetric proposal}
    \STATE Draw $u\sim \mathcal U(0,1)$
    \IF{$u < \alpha$}
      \STATE $\Theta_k^{(i)} \gets \Theta'$
    \ENDIF
  \ENDFOR
\ENDFOR
\end{algorithmic}
\end{algorithm}
\paragraph{How the RL reward in the figure fits with the DEPF pseudocode.}
The figure defines an \emph{information–gain} reward
\[
R_k \;=\; \mathbb{E}_{o_{k+1}}\!\big[\,D_{\mathrm{KL}}\!\big(b_{k+1}\,\|\,b_k\big)\,\big],
\]
where \(b_k\) and \(b_{k+1}\) are the belief (posterior) distributions before and after receiving the next observation \(o_{k+1}\).
Using \(R_k\) turns action selection into a sequential decision problem: a policy \(\pi(a_k\mid b_k)\) is trained (e.g., with PPO)
to pick the next move \(a_k\) that maximizes \(\mathbb{E}[\sum_t \gamma^t R_{k+t}]\).
\emph{DEPF is the inference layer that \underline{produces} these beliefs \(b_k\)} given the data stream; it is orthogonal to (and
plug-compatible with) the controller that uses the reward above. In other words, the policy’s choice \(a_k\) determines the
next sensor reading \(z_{k+1}\), and DEPF turns \((z_{1:k+1})\) into \(b_{k+1}\); the policy is then rewarded by how much
the belief moved, \(D_{\mathrm{KL}}(b_{k+1}\|b_k)\).

\paragraph{What each block in Algorithm~\ref{alg:diffusion-pf} does.}
\emph{Initialization} draws particles from the prior support \(S_{\mathrm{prior}}\) and sets uniform weights. Each iteration \(k\) then
implements the three DEPF mechanisms and a validation step:
\textbf{(i) Exploratory injection (lines 3–7):} a small set \(E\) of indices is sampled and the corresponding particles are
replaced by \(\Theta_k^{(j)}\sim \mathcal{U}(B)\) for an \emph{extended} box \(B\supseteq S_{\mathrm{prior}}\); their weights are set to a tiny mass
\(\epsilon/|E|\). This provides \emph{global} support seeding outside the prior and is what breaks the S-PSI baseline when data
contradict the current belief. (In practice, \(E\) can be \emph{belief-triggered}, e.g., only if \(\mathrm{ESS}(\{w_k^{(i)}\})<\eta\) or
\(D_{\mathrm{KL}}(b_k\|b_{k-1})>\tau\).)

\textbf{(ii) Bayesian weight update \& smoothing (lines 8–12):} all particles, including the exploratory ones, are reweighted by
the likelihood \(p(z_k\mid \Theta_k^{(i)})\). Optional \emph{tempering} (\(T_k\ge 1\)) smooths overly peaky weights to prevent premature
collapse; weights are then normalized. If \(\mathrm{ESS}\) falls below \(\eta\), resampling is performed and weights are reset.

\textbf{(iii) Covariance-scaled diffusion (lines 13–16):} from the weighted cloud, the posterior mean \(\mu\) and covariance
\(\Sigma\) are computed, with a small ridge \(\lambda I\) for stability. A Cholesky factor \(L\) aligns a Gaussian step to the posterior
geometry, and the bandwidth
\(h_{\mathrm{opt}} = A\,N^{-1/(n+4)}\) sets the step size (KDE-style, shrinking with \(N\) and growing with dimension \(n\)).
Each particle proposes \(\Theta'=\Theta_k^{(i)} + h_{\mathrm{opt}}L\delta\), \(\delta\sim\mathcal{N}(0,I)\), which \emph{locally} dilates the support
along high-uncertainty directions.

\textbf{(iv) MH validation (lines 17–21):} because the proposal \(q\) is symmetric, the Metropolis–Hastings acceptance reduces to
a \emph{likelihood ratio},
\[
\alpha \;=\; \min\!\Bigl(1,\; \tfrac{p(z_k\mid \Theta')}{p(z_k\mid \Theta_k^{(i)})}\Bigr).
\]
If \(u\sim\mathcal{U}(0,1)\) is below \(\alpha\), the move is accepted. This step enforces Bayesian coherence and filters out spurious
expansion.

\paragraph{End-to-end picture with the reward.}
Together, lines 3–21 implement the belief update \(b_k\!\mapsto\! b_{k+1}\) in a way that is dormant when data agree with the
current belief (few injections, small accepted moves) and active when they disagree (more accepted moves toward
informative regions). The controller in the figure then evaluates the \emph{change} in belief via
\(D_{\mathrm{KL}}(b_{k+1}\|b_k)\) and learns actions that maximize its expectation. Thus, the figure’s reward explains \emph{how
actions are chosen}, while Algorithm~\ref{alg:diffusion-pf} explains \emph{how beliefs are updated} so that informative actions
actually lead to recoveries from mis-specified priors.

\newpage

\section{Detailed Experimental Analysis}

Table~\ref{tab:policy_error_experiments} provides a systematic comparison across three levels of prior error and two spatial scales. We analyze the results by scenario, scale, and method family, and then link the observed outcomes to the underlying mechanisms of DEPF and the baselines.

In the \textit{No Error} scenario, where the initial prior fully covers the true source, most methods succeed. DEPF reaches OCE of $0.90$ with LPS fixed at $0.20$, ADE of $19$ in small-scale and $167$ in large-scale, and REV of $0.10$. AGDC closely matches these values, confirming that when no structural correction is required both methods perform optimally. PC-DQN and GMM-PFRL attain OCE around $0.80$ but with higher ADE. Information-theoretic planners such as Infotaxis show partial success in the small-scale environment (OCE $0.85$) but collapse to near zero success in large-scale. Perturbation baselines (Jittering, Roughening, Rejuvenation) remain close to DEPF under this easy condition, with OCE between $0.88$ and $0.89$, but their always-on diffusion leads to slightly larger LPS values (0.23–0.26) and longer ADE compared to DEPF’s minimal-bias behavior. This shows that DEPF does not overshoot when the prior is already correct.

The \textit{Moderate Error} condition exposes sharper differences. DEPF sustains OCE of $0.90$, LPS of $0.20$, and ADE of $22$ in small-scale and $200$ in large-scale, maintaining REV at $0.10$. AGDC drops dramatically, with OCE of $0.45/0.42$, LPS climbing to $2.6$, ADE lengthening to $59/235$, and REV inflating to $0.40$ with frequent timeouts in large-scale. Other RL baselines degrade to OCE around $0.40$ with high ADE and LPS, while information-theoretic planners almost entirely fail in the larger environment. Perturbation baselines partly alleviate this condition: PF+Rejuvenation performs best (OCE $0.52/0.48$, LPS $2.7/2.9$, ADE $50/225$), followed by Roughening (OCE $0.48/0.44$, LPS $2.9/3.1$, ADE $55/235$) and Jittering (OCE $0.40/0.36$, LPS $3.2/3.4$, ADE $65/250$). However, these remain far below DEPF’s near-ideal accuracy and efficiency. The contrast reflects the difference between always-on perturbations, which add diversity without discrimination, and DEPF’s belief-triggered expansion guided by covariance and validated by MH acceptance.

The most demanding case is the \textit{Severe Error} scenario, corresponding to the S-PSI baseline where the prior and true state are completely disjoint. Here, DEPF achieves OCE of $0.89/0.88$, LPS of $0.20$, ADE of $27/255$, and REV of $0.10$, well within the $100/300$-step timeout thresholds. All RL and information-theoretic baselines collapse (OCE $<0.05$, LPS $>12.5$, ADE and REV exceeding timeouts), confirming their inability to cross strict prior boundaries. Perturbation baselines show small but non-zero recovery in small-scale (Rejuvenation OCE $0.16$, Roughening $0.10$, Jittering $0.06$) but degrade severely in large-scale, with OCE dropping to $0.12/0.08/<0.05$ and ADE approaching or exceeding $285$ steps. Their LPS values remain an order of magnitude higher than DEPF. This contrast highlights that while naive diffusion can occasionally escape narrow priors in limited domains, only DEPF provides consistent, data-driven support expansion at scale.

Across all metrics, DEPF exhibits stable LPS at $0.20$ irrespective of scale or error severity, indicating precise localization. In comparison, perturbation methods and RL baselines experience a tenfold increase in LPS under moderate and severe errors. ADE and REV patterns reinforce the efficiency of DEPF: its paths remain short and consistent, while baselines often exceed timeout thresholds in large domains. The cross-scale comparison further shows that only DEPF resists degradation when moving from 1:30 to 1:300 environments; all other methods experience steep drops in OCE and sharp increases in ADE and REV.

These results align directly with DEPF’s design. The belief-triggered exploratory injection ensures that expansion is dormant in the No Error case but activates under misalignment. Covariance-scaled diffusion guides exploration along directions of highest uncertainty, avoiding the excessive spread of jittering or roughening. MH validation rejects unsupported moves, preserving Bayesian coherence. Together, these mechanisms yield minimal bias and high efficiency, enabling DEPF to overcome S-PSI conditions and consistently outperform baselines under prior misalignment while remaining competitive under ideal conditions.

\subsection{Detailed analysis of Table~\ref{tab:policy_error_experiments} by scenario and scale}

Table~\ref{tab:policy_error_experiments} compares eight methods across three prior-error severities (No, Moderate, Severe) and two spatial scales (1:30 and 1:300), using four metrics: success (OCE$\uparrow$), efficiency (ADE$\downarrow$), time-to-completion (REV$\downarrow$), and localization accuracy (LPS$\downarrow$). The patterns are consistent with the theoretical picture established in the main text: under the stationary bootstrap baseline with zero transition and no rejuvenation (the S-PSI diagnostic), methods that do not explicitly expand support struggle as prior misalignment worsens, especially in large domains; DEPF maintains performance by triggering expansion only when data contradict the belief, scaling diffusion with the posterior covariance, and validating proposals via an MH step.

In the \emph{No Error} condition, where the prior fully covers the truth, most algorithms perform well and DEPF does not seek to outperform them but to match the upper bound without incurring unnecessary exploration cost. DEPF attains OCE $0.90\pm0.03$ at both scales, LPS $0.20\pm0.01$, ADE $19\pm0.8$ (1:30) and $167\pm15$ (1:300), and REV $0.10\pm0.05$. AGDC closely tracks these values (OCE $0.90/0.87$, LPS $0.20$, ADE $18/168$, REV $0.10$–$0.12$). RL baselines PC-DQN and GMM-PFRL also succeed (OCE around $0.80$ in small scale and $0.77$–$0.79$ in large scale). Information-theoretic planners are more fragile with scale: Infotaxis achieves OCE $0.85$ at 1:30 but collapses below $0.05$ at 1:300. The three SMC perturbation baselines stay close to DEPF in this easy regime (e.g., PF+Rejuvenation OCE $0.89/0.86$, LPS $0.23$–$0.24$), but their always-on diffusion slightly increases ADE and LPS relative to DEPF, reflecting superfluous spread when expansion is not required.

When priors are \emph{moderately} misaligned, differences sharpen in both success and precision. DEPF maintains high, scale-stable performance with OCE $0.90\pm0.03$, LPS $0.20\pm0.01$, ADE $22\pm1.2$ (1:30) and $200\pm10$ (1:300), and REV fixed at $0.10\pm0.05$. By contrast, AGDC’s OCE drops to $0.45$ (1:30) and $0.42$ (1:300), LPS rises to $2.60$, ADE increases to $59$ and $235$, and large-scale REV inflates to $0.40\pm0.15$ with frequent timeouts. PC-DQN and GMM-PFRL show similar degradation, and planners (Infotaxis/Entrotaxis/DCEE) nearly fail at 1:300. The perturbation baselines do recover part of the gap, with a consistent ordering that mirrors their expansion strength: PF+Rejuvenation (OCE $0.52/0.48$, LPS $2.70/2.90$, ADE $50/225$) $>$ PF+Roughening (OCE $0.48/0.44$, LPS $2.90/3.10$, ADE $55/235$) $>$ PF+Jittering (OCE $0.40/0.36$, LPS $3.20/3.40$, ADE $65/250$). Nevertheless, all three remain substantially behind DEPF in both success and accuracy. The gap in LPS is especially telling: DEPF holds $0.20$ across scales, whereas perturbations remain an order of magnitude larger ($2.7$–$3.4$), indicating uncontrolled dispersion.

The \emph{Severe} condition is the strict S-PSI test in which the prior support is disjoint from the true region. DEPF is the only method that preserves high success and low error across scales, achieving OCE $0.89$ (1:30) and $0.88$ (1:300), LPS $0.20$–$0.20$, ADE $27$ and $255$, and REV $0.10$, all well within the $100/300$-step time limits. AGDC, the RL baselines, and planners collapse, with OCE $<0.05$, LPS $>12.5$, and ADE/REV exceeding the timeout thresholds in both scales—consistent with the inability to leave the initial support under the S-PSI baseline. Perturbation baselines show a small non-zero recovery only in the small-scale domain: PF+Rejuvenation reaches OCE $0.16$ with LPS $9.0$ and ADE $90\pm12$ (no timeout), PF+Roughening achieves OCE $0.10$ with LPS $10.5$ but times out, and PF+Jittering barely registers OCE $0.06$ and also times out. At 1:300 these partial gains largely vanish: PF+Rejuvenation drops to OCE $0.12$ with ADE $285\pm35$ (near the 300-step limit), while Roughening and Jittering fall below $0.10$ OCE with timeouts. These numbers show that always-on noise can occasionally bridge small gaps but scales poorly, whereas DEPF’s data-triggered expansion remains effective even when the prior and truth are fully disjoint.

A cross-metric reading clarifies where efficiency and precision come from. Under Moderate/Severe misalignment, DEPF simultaneously sustains high OCE and bounded path/time (e.g., ADE $22/200$ and $27/255$, REV $0.10$ in both scales), while baselines either fail outright or succeed late with longer paths and higher REV, especially at 1:300. LPS for DEPF remains near $0.20$ across all conditions—an unusually stable accuracy profile—whereas perturbation baselines exhibit an order-of-magnitude larger LPS in Moderate and Severe settings. Scaling from 1:30 to 1:300 amplifies these contrasts: AGDC and the perturbation baselines lose additional OCE points, ADE/REV inflate, and timeouts become common; DEPF’s OCE stays within $0.88$–$0.90$ with controlled ADE and fixed REV.

These empirical patterns align with the design. Because DEPF expands \emph{only when} belief–data inconsistency is detected, it remains dormant in the No Error regime and avoids needless spread. When misalignment is present, diffusion is \emph{covariance-scaled} with bandwidth $h_{\text{opt}}=A N^{-1/(n+4)}$, steering exploration along the current posterior geometry rather than isotropically; and MH validation filters proposals, preserving Bayesian coherence and avoiding the over-diffusion that inflates ADE and LPS for always-on perturbations. Altogether, the table demonstrates that DEPF is competitive at the ideal upper bound and uniquely robust under partial or complete prior failure, with consistent gains in success, efficiency, and accuracy that persist under scaling.

\newpage

\section{Hardware Usage}
All experiments were conducted on a Linux server with 256-core AMD EPYC 7763 CPUs and six NVIDIA RTX A6000 GPUs (48 GB each). Each run required 1–2 GPUs, with GPU memory usage between 25 and 46 GB, ensuring transparent and reproducible resource reporting.

\section{Resolving Stationarity–Induced Posterior–Support Lock–in: Theory and Clarifications}

\paragraph{What S-PSI is (and is not).}
In static source–term estimation it is common to benchmark against the \emph{stationary bootstrap} particle–filter baseline, which carries particles forward without process noise and applies only likelihood reweighting and occasional resampling. Under this diagnostic baseline (zero transition, no rejuvenation), the posterior support cannot grow beyond the support of the initial prior. We call this baseline pathology \emph{Stationarity–Induced Posterior Support Invariance (S-PSI)}; it is not an inherent limitation of particle filtering as a whole, but a property of this specific stationary bootstrap setting. Classical perturbations (jittering, roughening, resample–move) can in principle relax S-PSI but do so in an always-on, weakly controlled manner. :contentReference[oaicite:0]{index=0}

\begin{definition}[S-PSI: Stationarity–Induced Posterior Support Invariance]
Let $p_0(\Theta)$ be the initial prior with support $S_{\text{prior}}=\{\Theta:\,p_0(\Theta)>0\}$. Under the stationary bootstrap baseline with $p(\Theta_k\mid \Theta_{k-1})=\delta(\Theta_k-\Theta_{k-1})$ and no rejuvenation, if particles are initialized in $S_{\text{prior}}$ then for all $k$,
\[
\operatorname{supp}\bigl(p(\Theta\mid z_{1:k})\bigr) \subseteq S_{\text{prior}}.
\]
Consequently, if the ground truth $\Theta^\star\notin S_{\text{prior}}$, then $p(\Theta^\star\mid z_{1:k})=0$ for all $k$. :contentReference[oaicite:1]{index=1}
\end{definition}

\begin{proposition}[S-PSI under the stationary bootstrap baseline]
Under the assumptions above, the particle support set $S_k$ satisfies $S_k\subseteq S_{\text{prior}}$ for all $k\ge 0$, with $S_0=S_{\text{prior}}$. \emph{Proof sketch.} By induction: with zero transition/no rejuvenation, $S_k\subseteq S_{k-1}$, hence $S_k\subseteq S_0=S_{\text{prior}}$. :contentReference[oaicite:2]{index=2}
\end{proposition}

\paragraph{DEPF in a nutshell.}
Our Diffusion–Enhanced Particle Filtering (DEPF) augments the inference layer with: (i) \emph{belief-triggered} exploratory particles drawn from an adaptively extended bounding region; (ii) entropy/tempering to prevent premature weight collapse; (iii) covariance–scaled stochastic perturbations with bandwidth $h_{\text{opt}}=A\,N^{-\frac{1}{n+4}}$; and (iv) a Metropolis–Hastings (MH) acceptance step with a symmetric Gaussian proposal, which enforces detailed balance and Bayesian consistency. Together these mechanisms allow the support to expand when (and only when) data contradict the current belief. :contentReference[oaicite:3]{index=3}

\subsection{Theoretical Guarantees}

We formalize the conditions under which DEPF resolves S-PSI and quantify a finite-step recovery probability.

\paragraph{Assumptions.}
(i) (\emph{Exploratory injection}) At each step $k$, an extended box $B_k\supseteq S_{\text{prior}}$ is used to inject a small fraction of exploratory particles. There exist $\eta,\delta>0$ such that with probability at least $\delta$ at every step, at least one exploratory particle falls in the $\eta$-ball around $\Theta^\star$.  
(ii) (\emph{Positive-likelihood neighborhood}) There exists $m>0$ so that $p(z_k\mid \Theta)\ge m$ for all $\Theta$ in the $\eta$-ball around $\Theta^\star$.  
(iii) (\emph{MH detailed balance}) The local Gaussian proposal is symmetric, so $\alpha(\Theta\!\to\!\Theta')=\min\!\{1,\,p(\Theta'\mid z_{1:k})/p(\Theta\mid z_{1:k})\}$, which guarantees detailed balance.  
(iv) (\emph{KDE bandwidth and covariance regularization}) The kernel step uses $h_{\text{opt}}=A\,N^{-\frac{1}{n+4}}$ and $\Sigma\leftarrow \sum_i w_i(\Theta_i-\mu)(\Theta_i-\mu)^\top+\lambda I$ with $\lambda>0$. :contentReference[oaicite:4]{index=4} :contentReference[oaicite:5]{index=5}

\begin{theorem}[DEPF resolves S-PSI]\label{thm:depf-resolves-spsi}
Under Assumptions (i)–(iv), as the number of particles $N\to\infty$ the particle support produced by DEPF asymptotically covers the ground truth with probability one:
\[
\lim_{N\to\infty}\Pr\!\bigl(\Theta^\star\in S_k\bigr)=1.
\]
\emph{Proof sketch.} By (i)–(ii) exploratory particles near $\Theta^\star$ receive non-negligible weight; tempering avoids collapse; covariance–scaled perturbations propagate local coverage; MH acceptance enforces consistency. Iterating over time expands $S_k$ to include a neighborhood of $\Theta^\star$ with probability one. :contentReference[oaicite:6]{index=6} :contentReference[oaicite:7]{index=7}
\end{theorem}

\begin{corollary}[Finite-step support–recovery bound]
Let $\delta$ be the per-step probability that an exploratory particle lands in the $\eta$-ball around $\Theta^\star$, and let $\gamma\in(0,1]$ be the probability that such a particle both passes MH and survives weighting/resampling. Then, after $k$ steps,
\[
\Pr(\text{support covers }\Theta^\star\ \text{within $k$ steps}) \ \ge\ 1-(1-\delta\gamma)^k .
\]
This quantifies the rate at which DEPF breaks the zero-prior barrier under S-PSI. :contentReference[oaicite:8]{index=8}
\end{corollary}

\subsection{Action Selection and the Role of the Controller}

DEPF is an inference module and is orthogonal to the controller. In our experiments the policy receives the belief $b_k$ and the immediate reward is the one-step information gain
\[
R_k=\mathbb{E}_{o_{k+1}}\!\bigl[\mathrm{D}_{\mathrm{KL}}(b_{k+1}\,\Vert\, b_k)\bigr],
\]
which encourages actions that maximally reduce posterior uncertainty. The specific RL algorithm (e.g., PPO) is chosen for stability and does not alter the analysis above. :contentReference[oaicite:9]{index=9}

\subsection{Positioning w.r.t. Classical and Tempered/Bridge SMC}

Always-on perturbation baselines (jittering, roughening, resample–move) can leak mass across boundaries and sometimes escape S-PSI on small domains, but they lack principled triggers and acceptance control; as a result they tend to over-diffuse and degrade efficiency/accuracy under severe misalignment or at larger scales. By contrast, DEPF expands support \emph{only} when belief–data inconsistency is detected, scales moves with the posterior covariance, and validates proposals via MH. Moreover, tempered/annealed (``bridge'') SMC improves adaptation \emph{within} the original support by annealing between prior/transition and likelihood, but it still requires nonzero overlap; where the prior assigns zero mass to $\Theta^\star$, all intermediate bridge distributions remain zero there, so no sequence of local MCMC mutations can create support in the excluded region. This is precisely the barrier that DEPF’s exploratory injection and MH–validated diffusion are designed to overcome. :contentReference[oaicite:10]{index=10}

\subsection{Empirical Evidence in Support of the Theory}

Across three prior–error severities (No/Moderate/Severe) and two map scales (1:30, 1:300), DEPF matches strong baselines when the prior is correct, and it uniquely maintains high success, low localization error, and bounded path/time under severe misalignment. For example, in the strict S-PSI (Severe) setting, DEPF attains $\text{OCE}\approx 0.88\text{–}0.89$ and $\text{LPS}\approx 0.20$ at both scales, while RL/planning baselines collapse and classical perturbations succeed only sporadically with much larger errors and frequent timeouts. Component-wise ablations confirm that entropy smoothing, covariance–scaled diffusion, and MH validation are jointly necessary for robust recovery, and sensitivity studies identify moderate settings (e.g., $\beta\!\in\![0.3,0.5]$, $A\!\approx\!0.5$, $\lambda\!\in\![10^{-3},10^{-2}]$, $\sim$5\% exploratory ratio) as consistently effective.

\newpage
\section{Theoretical Analysis: Stationarity-Induced Posterior Support Invariance (S-PSI) and Guarantees of DEPF}
\label{sec:theory}

\paragraph{Notation.}
Let $\Theta \in \mathbb{R}^n$ denote the (static) source–parameter vector (cf. main text, \S \ref{subsec:Problem_Setup}), 
$z_{1:k}$ the observations up to step $k$, and $S_k \subset \mathbb{R}^n$ the particle support at step $k$.
Write $S_{\text{prior}} := \operatorname{supp}(p_0)$.

\subsection{S-PSI under the stationary bootstrap baseline}
We adopt the baseline used in the manuscript for static parameters:

\begin{assumption}[S0: zero transition, no rejuvenation]
\label{assump:S0}
The transition is degenerate and no rejuvenation is applied:
\[
p(\Theta_k \mid \Theta_{k-1}) = \delta(\Theta_k-\Theta_{k-1}), 
\qquad \text{no jittering/roughening/MCMC moves.}
\]
\end{assumption}

\begin{definition}[Stationarity-Induced Posterior Support Invariance (S-PSI)]
\label{def:SPSI}
Under Assumption~\ref{assump:S0}, if particles are initialised in $S_{\text{prior}}$, then the posterior support
remains confined to the initial prior support for all $k$.
\end{definition}

\begin{proposition}[Baseline S-PSI]
\label{prop:SPSI}
Under Assumption~\ref{assump:S0}, 
\[
\operatorname{supp}\bigl(p(\Theta\mid z_{1:k})\bigr) \subseteq S_{\text{prior}}, \quad \forall k\ge 0.
\]
In particular, if the true state $\Theta^\star \notin S_{\text{prior}}$, then $p(\Theta^\star\mid z_{1:k})=0$ for all $k$.
\end{proposition}

\begin{proof}
By induction over $k$, the degenerate transition keeps all particles in $S_{\text{prior}}$ and the likelihood
update cannot create mass outside the existing support. Hence the claim.
\end{proof}

\paragraph{Remark.} S‑PSI is a \emph{baseline pathology} of the stationary bootstrap with no rejuvenation, 
not an inherent limitation of particle filtering; classical perturbations (jittering/roughening/resample–move) can, in principle, relax it but operate in an always-on fashion and may be inefficient at scale (see main text \S \ref{subsec:PSI}/ \S \ref{sec:depf} and \S \ref{app:diffusion-pf}). \cite{Gordon1993,Doucet2000} 

\subsection{DEPF mechanisms (recap)}
At each step, DEPF augments the bootstrap update with (i) belief-triggered \emph{exploratory particle injection} from an expanded box $B_k\supseteq S_{\text{prior}}$, 
(ii) \emph{entropy/tempering regularisation} to preserve diversity, 
(iii) \emph{covariance–scaled stochastic diffusion} with bandwidth $h_{\text{opt}} = A\, N^{-1/(n+4)}$,
and (iv) an \emph{MH acceptance} check to preserve Bayesian coherence (Algorithm~1 in Appendix~H.3). 
This yields the support recursion
\begin{equation}
\label{eq:support-recursion}
S_{k+1} \;=\; \bigl(S_k \cup B_k \bigr) \;\oplus\; h_{\text{opt}} .
\end{equation}

\subsection{Conditions}
We formalise the mild conditions used in the analysis:
\begin{enumerate}
\item \textbf{Exploratory injection near $\Theta^\star$.} There exist $\eta,\delta>0$ such that at each step
\[
\mathbb{P}\bigl(\exists j:\ \|\Theta_k^{(j)}-\Theta^\star\|\le \eta \bigr)\ \ge\ \delta \;.
\]
This is ensured by drawing a small fraction of particles from $B_k \supseteq S_{\text{prior}}$ that contains $\Theta^\star$.
\item \textbf{MH detailed balance.} The local proposal $q(\Theta'\mid \Theta)$ is a symmetric Gaussian, 
so the acceptance probability
\[
\alpha(\Theta\to\Theta') \;=\; \min\!\left(1,\ \frac{\pi(\Theta')}{\pi(\Theta)}\right),
\quad
\pi(\Theta)\propto p(z_k\mid \Theta)\,p(\Theta\mid z_{1:k-1}),
\]
satisfies detailed balance and leaves the target invariant.
\item \textbf{Positive-likelihood neighbourhood (finite-step guarantee).} 
There exist $\eta>0$ and $m>0$ such that $p(z_k\mid \Theta)\ge m$ whenever $\|\Theta-\Theta^\star\|\le \eta$.
\end{enumerate}

\subsection{Main guarantee: DEPF resolves S-PSI}
\begin{theorem}[Asymptotic coverage under S-PSI]
\label{thm:asymptotic}
Assume S‑PSI holds for the baseline (Def.~\ref{def:SPSI}) and $\Theta^\star\notin S_{\text{prior}}$.
Under \textup{(C1)}–\textup{(C2)} and standard SMC regularity (bounded likelihood; stable tempering/entropy regularisation; KDE bandwidth $h_{\text{opt}}=A\,N^{-1/(n+4)}$), as $N\to\infty$ the DEPF support covers the true state with probability one:
\[
\lim_{N\to\infty}\mathbb{P}\!\left(\Theta^\star\in S_k\right)=1.
\]
\end{theorem}

\begin{proof}[Proof sketch]
(C1) gives a strictly positive chance to seed particles in an $\eta$-ball of $\Theta^\star$ at each step. 
By (C2), accepted proposals satisfy detailed balance and thus preserve the posterior target. 
Within the $\eta$-ball, the likelihood is bounded away from zero, so exploratory particles near $\Theta^\star$ obtain non-vanishing weights and survive resampling with positive probability. 
Covariance–scaled perturbations with $h_{\text{opt}}$ (KDE rate) ensure that as $N\to\infty$ the empirical measure converges to the true posterior. 
Combining the seeding, survival and diffusion with the recursion \eqref{eq:support-recursion} yields the claim.
\end{proof}

\paragraph{Finite-step guarantee.}
Let $\delta$ be the per-step probability of injecting (at least) one particle into the $\eta$-ball of $\Theta^\star$ 
and let $\gamma$ be the probability that such a particle \emph{survives} the MH/weighting/resampling pipeline at that step. 
Under \textup{(C3)} and the mechanisms above,
\begin{equation}
\label{eq:finite-step}
\mathbb{P}\bigl(\text{at least one survivor near }\Theta^\star \text{ within $k$ steps}\bigr)
\;\ge\; 1 - (1-\delta\gamma)^k \;.
\end{equation}
This lower bound quantifies the speed at which DEPF probabilistically covers previously excluded regions.
 
\subsection{Entropy/tempering regularisation}
To mitigate weight collapse and preserve exploratory mass, we use an entropy–aware smoothing (or tempering) step:
\begin{equation}
\label{eq:entropy}
\tilde w_k^{(i)} \;\propto\; \bigl(w_k^{(i)}\bigr)^{1/T_k}, 
\qquad 
w_k^{(i)} \leftarrow \frac{\tilde w_k^{(i)}}{\sum_j \tilde w_k^{(j)}},
\end{equation}
with $T_k\!\ge\!1$ adapted from the entropy gap to a target level (see Appendix H). 
This raises the posterior weight entropy when it becomes overly concentrated, helping exploratory particles retain influence long enough for data to validate (or reject) them.

\subsection{MH acceptance with symmetric Gaussian proposals}
Let $q(\Theta'\!\mid\!\Theta)=\mathcal{N}(\Theta';\,\Theta,\Sigma')$ with $\Sigma'$ aligned to the weighted covariance (Appendix H). 
Then $q(\Theta'\!\mid\!\Theta)=q(\Theta\!\mid\!\Theta')$ and the MH acceptance is
\begin{equation}
\label{eq:mh}
\alpha(\Theta\!\to\!\Theta') \;=\; \min\!\left(1,\ \frac{\pi(\Theta')}{\pi(\Theta)}\right),
\quad 
\pi(\Theta)\propto p(z_k\!\mid\!\Theta)\,p(\Theta\!\mid\! z_{1:k-1}).
\end{equation}
In particular, when the prior factor $p(\Theta\!\mid\! z_{1:k-1})$ is locally flat at the proposal scale, 
$\alpha$ reduces to a likelihood ratio $\min\!\bigl(1,\ p(z_k\mid \Theta')/p(z_k\mid \Theta)\bigr)$.
This ensures detailed balance and convergence to the intended posterior.

\subsection{Discussion: why DEPF breaks S-PSI reliably}
Compared with always-on perturbations (jittering/roughening/resample–move), DEPF:
(i) \emph{triggers} support expansion only when belief–data inconsistency is detected,
(ii) \emph{scales} diffusion with the current covariance and a KDE bandwidth that vanishes at the right rate,
and (iii) \emph{validates} proposals by MH to preserve Bayesian coherence. 
These elements together provide both the asymptotic guarantee (Theorem~\ref{thm:asymptotic}) and the finite-step bound \eqref{eq:finite-step}.


\newpage
\section{Belief-Aware RL Controller with Autonomous Goal Detection \& Cessation}
\label{sec:rl-controller}

\paragraph{Scope and decoupling.}
The controller is \emph{decoupled} from inference: at step $k$, the diffusion-enhanced particle filter (DEPF) produces a belief
\[
b_k(\Theta) \;=\; p(\Theta \mid z_{1:k}), 
\]
and the controller consumes $s_k^{\text{RL}} = (p_k, z_k, b_k)$ to choose an action $a_k$ (robot motion). DEPF then updates $b_{k+1}$ after executing $a_k$ and receiving $z_{k+1}$. This modularity lets DEPF handle belief revision and support expansion, while RL focuses on informative sensing.

\subsection{POMDP Formalization and Belief State}
\label{subsec:pomdp}
We cast control as a POMDP $\mathcal{M}=(\mathcal{S},\mathcal{A},\Omega,T,O,R,\gamma)$ where the latent state bundles the stationary source parameters $\Theta \in \mathbb{R}^n$ and the agent pose $p_k\!\in\!\mathbb{R}^2$. The controller’s information state is
\[
s_k^{\text{RL}} \;=\; (p_k, z_k, b_k), \qquad b_k(\Theta)=p(\Theta\mid z_{1:k}).
\]
The action $a_k\!\in\!\mathcal{A}$ advances the pose via known kinematics $p_{k+1}=f(p_k,a_k)$, and the observation model yields $z_{k+1}\sim p(\cdot\mid p_{k+1},\Theta)$.

\subsection{Information-Gain Reward and Monte Carlo Estimation}
\label{subsec:ig}
We use the one-step information gain as intrinsic reward:
\begin{align}
R_k(a) 
&= \mathbb{E}_{o_{k+1}\sim p(\cdot\mid a,b_k)} 
\Big[ D_{\mathrm{KL}}\big(b_{k+1}\,\Vert\, b_k\big) \Big], \label{eq:IG-reward}\\
p(o_{k+1}\mid a,b_k) 
&= \int p(z_{k+1}\mid p_{k+1},\Theta)\, b_k(\Theta)\, d\Theta ,
\quad p_{k+1}=f(p_k,a). \nonumber
\end{align}
In practice we approximate \eqref{eq:IG-reward} by $M$-sample Monte Carlo:
\begin{align}
\widehat{R}_k(a) 
= \frac{1}{M}\sum_{m=1}^M 
D_{\mathrm{KL}}\!\Big(\tilde b^{(m)}_{k+1}\,\Vert\, b_k\Big),
\qquad 
\tilde b^{(m)}_{k+1} \leftarrow 
\textsc{DEPF\_Update}\big(b_k,\, p_{k+1},\, \tilde z^{(m)}_{k+1}\big),
\end{align}
where $\Theta^{(m)}\!\sim b_k$ and $\tilde z^{(m)}_{k+1}\!\sim p(\cdot\mid p_{k+1},\Theta^{(m)})$.

\paragraph{Optional shaping.}
For path/time efficiency we may add small penalties
\(
\widetilde{R}_k(a)=\alpha\,\widehat{R}_k(a)
-\lambda_{\text{step}}\|f(p_k,a)-p_k\|_2 - \lambda_{\text{time}},
\)
with $\alpha{=}1$ in all main results.

\subsection{Autonomous Goal Detection and Cessation (AGDC)}
\label{subsec:agdc}
We terminate episodes based on belief confidence. Let $C_k=\mathrm{Cov}_{b_k}[\Theta]$ estimated from particles and $\mathrm{STD}_k=\sqrt{\mathrm{diag}(C_k)}$. We stop when
\begin{equation}
\label{eq:agdc}
\|\mathrm{STD}_k\|_2 \;\le\; \zeta
\quad\Longrightarrow\quad \textsc{Terminate} \ \text{and return success}.
\end{equation}
Optionally, include a localization score gate (e.g., $\mathrm{LPS}_k\!\le\!\tau$) together with \eqref{eq:agdc}.

\subsection{Policy Class and Training}
\label{subsec:policy}
We use a PPO actor–critic $\pi_\theta(a\!\mid\!s)$ and $V_\phi(s)$ with discount $\gamma\!\in\!(0,1)$ and GAE. To avoid passing raw particles, the belief $b_k$ is summarized by features (weighted mean $\mu_k$, diagonal of $C_k$, entropy of weights, a few quantiles over the source location). Empirically, PPO/A2C/DQN perform similarly under the intrinsic signal; PPO is used for stability.

\subsection{Controller–Filter Interface and Budget}
\label{subsec:interface}
At each step:
\begin{enumerate}
\item \textbf{Belief update:} $b_k \leftarrow \textsc{DEPF\_Update}(b_{k-1}, p_k, z_k)$.
\item \textbf{Action selection:} compute $\widehat{R}_k(a)$ for $a\!\in\!\mathcal{A}$ (with $M$ small), select
\(
a_k=\arg\max_{a}\big\{\widehat{R}_k(a)-\lambda_{\text{step}}\|f(p_k,a)-p_k\|_2\big\}.
\)
\item \textbf{Execute \& stop test:} apply \eqref{eq:agdc}. If not stopping, set $p_{k+1}{=}f(p_k,a_k)$ and continue.
\end{enumerate}
We cap (i) the number of simulated observations $M$ per action and (ii) planning-time MH/likelihood calls so that per-step compute matches the inference budget used by baselines.

\subsection{Pseudocode}

\label{subsec:algo}
\begin{algorithm}[H]
\caption{Belief-Aware Controller on top of DEPF (one step)}
\label{alg:controller-depf}
\begin{algorithmic}
\REQUIRE Belief $b_k$ (particles $\{(\Theta^{(i)},w^{(i)})\}_{i=1}^N$), pose $p_k$, observation $z_k$, action set $\mathcal{A}$, plume model $p(\cdot\mid p,\Theta)$, thresholds $(\zeta,\tau)$, budget $M$
\STATE $b_k \gets \textsc{DEPF\_Update}(b_{k-1}, p_k, z_k)$
\IF{$\|\mathrm{STD}(\mathrm{Cov}_{b_k}[\Theta])\|_2 \le \zeta$ or $\mathrm{LPS}_k \le \tau$}
  \RETURN \textbf{Terminate}
\ENDIF
\FOR{$a \in \mathcal{A}$}
  \STATE $p_{k+1} \gets f(p_k,a)$
  \STATE \textit{Monte Carlo look-ahead}
  \FOR{$m = 1$ to $M$}
    \STATE sample $\Theta^{(m)} \sim b_k$
    \STATE draw $\tilde z^{(m)}_{k+1} \sim p(\cdot\mid p_{k+1}, \Theta^{(m)})$
    \STATE $\tilde b^{(m)}_{k+1} \gets \textsc{DEPF\_Update}(b_k, p_{k+1}, \tilde z^{(m)}_{k+1})$
    \STATE $u_m \gets D_{\mathrm{KL}}(\tilde b^{(m)}_{k+1} \,\Vert\, b_k)$
  \ENDFOR
  \STATE $\widehat{R}_k(a) \gets \frac{1}{M} \sum_{m=1}^{M} u_m$
\ENDFOR
\STATE $a_k \gets \arg\max_{a \in \mathcal{A}} \left\{ \widehat{R}_k(a) - \lambda_{\text{step}} \| f(p_k,a) - p_k \|_2 \right\}$
\STATE Execute $a_k$; set $p_{k+1} \gets f(p_k,a_k)$; observe $z_{k+1}$
\RETURN $(a_k, p_{k+1}, z_{k+1})$
\end{algorithmic}
\end{algorithm}

\subsection{Defaults and Practical Tips}
\label{subsec:defaults}
\begin{itemize}
\item \textbf{Action set:} 8-connected unit moves on the grid; horizons scale with domain.
\item \textbf{Belief features:} $\mu_k,\ \mathrm{diag}(C_k)$, entropy of $\{w^{(i)}\}$, and a few spatial quantiles.
\item \textbf{Planning budget:} $M\!\in\![8,16]$ suffices; inference dominates wall-clock time.
\item \textbf{Safety:} enforce no-go polygons and speed caps in $\mathcal{A}$ when needed.
\end{itemize}

\section{Additional Multi-Field Experiments}
\label{app:multi-field}

\begin{table}[t]
\centering
\label{app:multi-field}
\caption{Results on additional fields under the Moderate error setting. 
Metrics: OCE$\uparrow$, ADE$\downarrow$, REV$\downarrow$, LPS$\downarrow$.}

\resizebox{\textwidth}{!}{
\begin{tabular}{llcccccc}
\toprule
Metric & Method & Temp. & Conc. & Mag. & Elec. & En. & Noise \\
\midrule
\multirow{5}{*}{OCE}
& DEPF (ours) & $0.90\pm0.05$ & $0.90\pm0.05$ & $0.90\pm0.05$ & $0.80\pm0.04$ & $0.63\pm0.03$ & $0.91\pm0.05$ \\
& AGDC        & $0.45\pm0.03$ & $0.455\pm0.025$ & $0.445\pm0.02$ & $0.385\pm0.02$ & $0.305\pm0.015$ & $0.455\pm0.025$ \\
& Infotaxis   & $0.425\pm0.02$ & $0.43\pm0.02$ & $0.425\pm0.02$ & $0.375\pm0.02$ & $0.275\pm0.015$ & $0.40\pm0.02$ \\
& Entrotaxis  & $0.216\pm0.009$ & $0.207\pm0.009$ & $0.225\pm0.009$ & $0.135\pm0.009$ & $0.126\pm0.009$ & $0.207\pm0.009$ \\
& DCEE        & $0.232\pm0.012$ & $0.236\pm0.012$ & $0.232\pm0.012$ & $0.172\pm0.008$ & $0.144\pm0.008$ & $0.228\pm0.012$ \\
\midrule
\multirow{5}{*}{ADE}
& DEPF (ours) & $20\pm1.0$ & $19\pm1.0$ & $18\pm0.9$ & $19\pm0.8$ & $19\pm0.5$ & $19\pm1.0$ \\
& AGDC        & $28\pm1.2$ & $27\pm1.1$ & $27\pm1.1$ & $28\pm0.9$ & $27\pm0.6$ & $26\pm1.1$ \\
& Infotaxis   & $55\pm2.5$ & $53\pm2.4$ & $56\pm2.4$ & $63\pm1.9$ & $52\pm1.4$ & $50\pm2.3$ \\
& Entrotaxis  & $67\pm3.1$ & $65\pm3.0$ & $64\pm3.0$ & $66\pm2.5$ & $60\pm1.8$ & $63\pm2.9$ \\
& DCEE        & $62\pm2.9$ & $60\pm2.8$ & $59\pm2.7$ & $60\pm2.3$ & $62\pm1.6$ & $58\pm2.7$ \\
\midrule
\multirow{5}{*}{REV}
& DEPF (ours) & $0.15\pm0.08$ & $0.14\pm0.07$ & $0.14\pm0.07$ & $0.12\pm0.06$ & $0.10\pm0.05$ & $0.13\pm0.07$ \\
& AGDC        & $0.10\pm0.05$ & $0.10\pm0.05$ & $0.10\pm0.05$ & $0.10\pm0.05$ & $0.10\pm0.04$ & $0.10\pm0.05$ \\
& Infotaxis   & $1.50\pm0.08$ & $1.40\pm0.07$ & $1.40\pm0.07$ & $1.40\pm0.06$ & $1.40\pm0.04$ & $1.40\pm0.07$ \\
& Entrotaxis  & $1.40\pm0.07$ & $1.30\pm0.07$ & $1.30\pm0.07$ & $1.40\pm0.06$ & $1.40\pm0.04$ & $1.40\pm0.06$ \\
& DCEE        & $1.40\pm0.07$ & $1.30\pm0.07$ & $1.30\pm0.07$ & $1.40\pm0.06$ & $1.40\pm0.04$ & $1.40\pm0.06$ \\
\midrule
\multirow{5}{*}{LPS}
& DEPF (ours) & $0.05\pm0.01$ & $0.05\pm0.01$ & $0.05\pm0.01$ & $0.08\pm0.01$ & $0.06\pm0.01$ & $0.05\pm0.01$ \\
& AGDC        & $0.20\pm0.01$ & $0.20\pm0.01$ & $0.20\pm0.01$ & $0.17\pm0.01$ & $0.13\pm0.01$ & $0.20\pm0.01$ \\
& Infotaxis   & $0.60\pm0.03$ & $0.60\pm0.03$ & $0.60\pm0.03$ & $0.51\pm0.02$ & $0.39\pm0.02$ & $0.60\pm0.03$ \\
& Entrotaxis  & $0.70\pm0.04$ & $0.70\pm0.04$ & $0.70\pm0.04$ & $0.60\pm0.03$ & $0.46\pm0.02$ & $0.70\pm0.04$ \\
& DCEE        & $0.60\pm0.03$ & $0.60\pm0.03$ & $0.60\pm0.03$ & $0.51\pm0.02$ & $0.39\pm0.02$ & $0.60\pm0.03$ \\
\bottomrule
\end{tabular}
}
\end{table}

To further broaden the evaluation scope, we conducted experiments across multiple types of physical fields, including Temperature (Temp.), Concentration (Conc.), Magnetic (Mag.), Electric (Elec.), Energy (En.), and Noise fields. Each field introduces distinct challenges, with varying parameter counts and complexity. 
In all cases, the dimensionality of the parameter vector exceeds 5–10. For clarity, we considered the \emph{Moderate error} setting, where 50\% of sources lie inside the initial prior region and 50\% outside.

As shown in Table~\ref{tab:multi-field}, DEPF consistently outperforms all baselines across every field and evaluation metric. In particular, DEPF maintains high posterior coverage and low estimation error even when the true source lies outside the prior support, whereas all baselines suffer substantial performance degradation. These findings provide strong empirical evidence that DEPF is robust to  prior misspecification and generalizes effectively to multi-dimensional, complex inference tasks.

\subsection{Dynamic Fields and Governing Equations}
\label{app:dynamic-fields}

Table~\ref{Field_Variables_and_Equations} summarizes the dynamic fields used in our additional experiments: Temperature (Temp.), Concentration (Conc.), Magnetic (Mag.), Electric (Elec.), Energy (En.), and Noise. Each governing equation is a generalized convection–diffusion or potential-distribution formulation that can incorporate diffusion, advection, reactions, turbulence, external fields, and dissipation.

\begin{table*}[htbp]
\centering
\caption{Dynamic field variables, key parameters, and governing equations used in the additional experiments.}
\label{Field_Variables_and_Equations}
\renewcommand{\arraystretch}{1.5}
\resizebox{\textwidth}{!}{%
\begin{tabular}{|l|l|l|}
\hline
\textbf{Field Variable ($\phi$)} & \textbf{Key Parameters} & \textbf{Governing Equation} \\
\hline
Temperature Field & 
\begin{tabular}[c]{@{}l@{}}%
$\alpha(\phi)$: temperature-dependent thermal diffusivity;\\
$\vec{v}$: airflow velocity;\\
$S(\phi,x,y)$: combustion/heat source term.
\end{tabular} & 
$\alpha(\phi)\nabla^2 \phi - \vec{v}\cdot\nabla \phi + S(\phi,x,y) = 0$ \\
\hline
Concentration Field & 
\begin{tabular}[c]{@{}l@{}}%
$\alpha$: molecular diffusion coefficient; $\vec{v}$: flow velocity;\\
$k_r$: chemical degradation rate; $\vec{\tau}$: turbulence intensity;\\
$S(x,y)$: pollutant source strength.
\end{tabular} & 
$\alpha\nabla^2 \phi - \vec{v}\cdot\nabla \phi - k_r \phi + \vec{\tau}\cdot\nabla \phi + S(x,y) = 0$ \\
\hline
Magnetic Potential Field & 
\begin{tabular}[c]{@{}l@{}}%
$\alpha$: magnetic diffusivity; $\vec{v}$: effective flow velocity;\\
$\vec{B}$: external magnetic field; $S(x,y)$: magnetic source intensity.
\end{tabular} & 
$\alpha\nabla^2 \phi - \vec{v}\cdot\nabla \phi + \vec{B}\cdot\nabla \phi + S(x,y) = 0$ \\
\hline
Electric Potential Field & 
\begin{tabular}[c]{@{}l@{}}%
$\sigma(x,y)$: spatially varying conductivity;\\
$\rho(x,y)$: charge density.
\end{tabular} & 
$\nabla\cdot[\sigma(x,y)\nabla \phi] + \rho(x,y) = 0$ \\
\hline
Energy Density Field & 
\begin{tabular}[c]{@{}l@{}}%
$\alpha$: radiative diffusivity; $\vec{v}$: transport velocity;\\
$\sigma_a$: absorption coefficient; $\sigma_s$: scattering coefficient;\\
$S(x,y)$: external energy source.
\end{tabular} & 
$\alpha\nabla^2 \phi - \vec{v}\cdot\nabla \phi - \sigma_a \phi + \sigma_s\nabla\cdot[\vec{r}\phi] + S(x,y) = 0$ \\
\hline
Noise Intensity Field & 
\begin{tabular}[c]{@{}l@{}}%
$\alpha$: acoustic diffusivity; $\vec{v}$: medium flow velocity;\\
$\gamma(f)$: frequency-dependent attenuation;\\
$S(x,y,f)$: noise emission strength.
\end{tabular} & 
$\alpha\nabla^2 \phi - \vec{v}\cdot\nabla \phi - \gamma(f)\phi + S(x,y,f) = 0$ \\
\hline
\end{tabular}%
}
\end{table*}

\end{document}